\newif\ifisarxiv
\isarxivtrue

\newif\ifpicture
\picturetrue

\ifisarxiv

\documentclass{article}

\usepackage{fullpage}
\usepackage[numbers]{natbib}

\usepackage[utf8]{inputenc} 
\usepackage[T1]{fontenc}    
\usepackage{hyperref}       
\usepackage{url}            
\usepackage{booktabs}       
\usepackage{amsfonts}       
\usepackage{nicefrac}       
\usepackage{microtype}      
\usepackage{xcolor}         
\usepackage{amsmath, amssymb, graphicx, url, algorithm2e}
\usepackage{times}
\usepackage{subfigure}
\usepackage{tikz,pgfplots}
\usepackage{enumitem}
\usepackage{float}
\pgfplotsset{compat=newest}
\usepackage{xcolor}
\hypersetup{
	colorlinks,
	linkcolor={red!40!gray},
	citecolor={blue!40!gray},
	urlcolor={blue!70!gray}
}

\def\nnz{{\mathrm{nnz}}}
\def\tv{{\mathrm{tv}}}
\def\Lip{{\mathrm{Lip}}}

\def\CV{{\mathrm{CV}}}

\def\Ec{\mathcal E}
\def\Ac{\mathcal A}

\def\sbt{{\tilde \s}}

\newcommand{\Er}{\mathrm{Er}}
\newif\ifDRAFT
\DRAFTtrue
\ifDRAFT
\newcommand{\marrow}{\marginpar[\hfill$\longrightarrow$]{$\longleftarrow$}}
\newcommand{\niceremark}[3]
   {\textcolor{red}{\textsc{#1 #2:} \marrow\textsf{#3}}}
\newcommand{\ken}[2][says]{\niceremark{Ken}{#1}{#2}}

\newcommand{\michael}[2][says]{\niceremark{Michael}{#1}{#2}}
\newcommand{\michal}[2][says]{\niceremark{Michal}{#1}{#2}}
\newcommand{\feynman}[2][says]{\niceremark{Feynman}{#1}{#2}}
\else
\newcommand{\ken}[1]{}
\newcommand{\michael}[1]{}
\newcommand{\michal}[1]{}
\newcommand{\feynman}[1]{}
\fi

\def\poly{{\mathrm{poly}}}

\def\Sigmab{\mathbf{\Sigma}}

\def\Sigmabt{\widetilde{\Sigmab}}
\def\S{\mathbf{S}}
\def\T{\mathbf{T}}
\def\xt{\tilde{x}}
\def\xbt{\tilde{\x}}

\def\zbt{\tilde{\z}}

\def\s {\mathbf{s}}

\def\Xbt{\widetilde{\X}}

\def\g{{\mathbf{g}}}

\def\Nc{\mathcal{N}}

\def\Cc{\mathcal{C}}

\def\r{\mathbf r}

\def\H{\mathbf H}
\def\Hbh{\widehat{\H}}

\def\Kc{\mathcal{K}}

\def\Q{\mathbf Q}

\def\c{{n-d\choose s-d}}

\ifx\BlackBox\undefined
\newcommand{\BlackBox}{\rule{1.5ex}{1.5ex}}  
\fi
\DeclareMathOperator*{\argmin}{\mathop{\mathrm{argmin}}}

\def\x{\mathbf x}
\def\y{\mathbf y}

\def\z{\mathbf z}
\def\a{\mathbf a}
\def\b{\mathbf b}
\def\w{\mathbf w}
\def\v{\mathbf v}

\def\wbh{\widehat{\mathbf w}}

\def\e{\mathbf e}
\def\zero{\mathbf 0}
\def\one{\mathbf 1}
\def\u{\mathbf u}

\def\X{\mathbf X}

\def\B{\mathbf B}
\def\A{\mathbf A}
\def\C{\mathbf C}
\def\U{\mathbf U}
\def\Ubt{\widetilde{\mathbf U}}

\def\G{\mathbf G}
\def\D{\mathbf D}
\def\V{\mathbf V}
\def\M{\mathbf M}

\def\St{\widetilde{\S}}

\def\Sc{\mathcal{S}}
\def\Fc{\mathcal{F}}

\def\Bbt{\widetilde{\mathbf{B}}}

\def\Z{\mathbf Z}

\def\I{\mathbf I}

\def\A{\mathbf A}

\def\E{\mathbb E}
\def\R{\mathbb R}

\def\Pr{\mathrm{Pr}}
\def\tr{\mathrm{tr}}

\def\rank{\mathrm{rank}}

\let\origtop\top
\renewcommand\top{{\scriptscriptstyle{\origtop}}} 

\definecolor{silver}{cmyk}{0,0,0,0.3}
\definecolor{yellow}{cmyk}{0,0,0.9,0.0}
\definecolor{reddishyellow}{cmyk}{0,0.22,1.0,0.0}
\definecolor{black}{cmyk}{0,0,0.0,1.0}
\definecolor{darkYellow}{cmyk}{0.2,0.4,1.0,0}
\definecolor{orange}{cmyk}{0.0,0.7,0.9,0}
\definecolor{darkSilver}{cmyk}{0,0,0,0.1}
\definecolor{grey}{cmyk}{0,0,0,0.5}
\definecolor{darkgreen}{cmyk}{0.6,0,0.8,0}

\ifx\proof\undefined
\newenvironment{proof}{\par\noindent{\bf Proof\ }}{\hfill\BlackBox\\[2mm]}
\fi

\ifx\theorem\undefined
\newtheorem{theorem}{Theorem}
\fi

\ifx\example\undefined
\newtheorem{example}{Example}
\fi

\ifx\condition\undefined
\newtheorem{condition}{Condition}
\fi
\ifx\property\undefined
\newtheorem{property}{Property}
\fi

\ifx\lemma\undefined
\newtheorem{lemma}{Lemma}
\fi

\ifx\proposition\undefined
\newtheorem{proposition}{Proposition}
\fi

\ifx\remark\undefined
\newtheorem{remark}{Remark}
\fi

\ifx\corollary\undefined
\newtheorem{corollary}{Corollary}
\fi

\ifx\definition\undefined
\newtheorem{definition}{Definition}
\fi

\ifx\conjecture\undefined
\newtheorem{conjecture}{Conjecture}
\fi

\ifx\axiom\undefined
\newtheorem{axiom}{Axiom}
\fi

\ifx\claim\undefined
\newtheorem{claim}{Claim}
\fi

\ifx\assumption\undefined
\newtheorem{assumption}{Assumption}
\fi

\ifx\condition\undefined
\newtheorem{condition}{Condition}
\fi

\title{Algorithmic Gaussianization through Sketching:\\
  Converting Data into Sub-gaussian Random Designs
}

\author{%
  \textbf{Micha{\l } Derezi\'{n}ski}\\
  Department of Electrical Engineering \& Computer Science\\
  University of Michigan\\
  \texttt{derezin@umich.edu}\\
}

\else
\documentclass[final,12pt]{colt2023} 

\title[Algorithmic Gaussianization through Sketching]{Algorithmic Gaussianization through Sketching:\\ Converting Data into Sub-gaussian Random Designs}
\usepackage{times}
\usepackage{tikz,pgfplots}
\usepackage{enumitem}
\usepackage{wrapfig}
\pgfplotsset{compat=newest}

\coltauthor{%
 \Name{Micha{\l} Derezi\'nski} \Email{derezin@umich.edu}\\
 \addr University of Michigan
}

\fi

\begin{document}

\maketitle

\begin{abstract}%
Algorithmic Gaussianization is a phenomenon that can arise when using
randomized sketching or sampling methods to produce smaller
representations of large datasets: For certain tasks, these sketched
representations have been observed to
exhibit many robust performance characteristics that are known to
occur when a data sample comes from a sub-gaussian random design,
which is a powerful statistical model of data distributions. However, this
phenomenon has only been studied for specific tasks and 
metrics, or by relying on computationally expensive methods. We
address this by providing an algorithmic framework for
gaussianizing data using sparse sketching operators, proving that it is
possible to efficiently construct data sketches that are nearly
indistinguishable (in terms of total variation distance) from
sub-gaussian random designs. In particular, relying on a recently
introduced sketching technique called Leverage Score Sparsified (LESS)
embeddings, we show that one can construct an $n\times d$ sketch of an
$N\times d$ matrix 
$\A$, where $n\ll N$, that is nearly indistinguishable from a sub-gaussian
design, in time $O(\mathrm{nnz}(\A)\log N + nd^2)$, where $\mathrm{nnz}(\A)$ is
the number of non-zero entries in $\A$. As a
consequence, strong statistical guarantees and precise asymptotics
available for the estimators produced from sub-gaussian
designs (e.g., for least squares and Lasso regression, covariance estimation,
low-rank approximation, etc.) can be straightforwardly adapted to our
sketching framework. We illustrate this with a new approximation
guarantee for sketched least squares, among other examples. The key
technique that enables our analysis is a novel variant of the
Hanson-Wright inequality on the concentration of random quadratic
forms, which we establish for random vectors that arise from sparse sketches.
\end{abstract}

\ifisarxiv\else
\begin{keywords}%
Sketching, Least squares, Randomized Numerical Linear Algebra,
Sub-gaussianity
\end{keywords}
\fi

\section{Introduction}
\label{s:intro}
In a standard statistical learning setup, we are given a sample of
i.i.d.~points $(\x_1,y_1),...,(\x_n,y_n)$, and our
goal is 
to perform an estimation or prediction task. For example, in linear
regression, we aim to learn a linear model $\w^*$ from labels/responses $y_i =
\x_i^\top\w^*+\xi_i$, where $\xi_i$ represents the noise. Naturally, the
performance of prediction models for linear regression, such as
ordinary least squares (OLS) and
Lasso, depends greatly on the properties of the sample distribution, as well
as on the distribution of the noise $\xi$, e.g., whether they are
heavy-tailed or not. To that end, there is extensive
literature (see Section~\ref{s:related-work}) which provides precise analysis of statistical and machine
learning models under strong distributional assumptions on the data,
one of the most common assumptions being  
Gaussianity and sub-gaussianity. When, the data and noise exhibit small sub-gaussian tails
(e.g., by having a constant sub-gaussian Orlicz norm), then estimators such as OLS and Lasso, as
well as model selection methods such as cross-validation, exhibit
provably good, even optimal, performance. In modern learning tasks these
distributional assumptions are rarely met. On the other hand, we can often benefit
from a great abundance of cheap data available in many
domains. To that end, we ask: Can we leverage this data abundance to
algorithmically introduce sub-gaussianity into a data
distribution, and when is it practical?

A key motivation in this context is randomized sketching,
which is useful in situations when running a learning/estimation algorithm 
directly on the entire dataset is computationally
prohibitive. Instead, we produce a smaller sketch of 
the data, e.g., via importance sampling or random projections, and use
it as a surrogate data sample. To illustrate this, suppose that our
goal is to solve a least squares regression task $(\A,\b)$:
\begin{align*}
  \text{Find}\qquad \w^* = \argmin_{\w}L(\w)\qquad\text{for}\quad L(\w)=\|\A\w-\b\|^2,
\end{align*}
where $\A$ is a large tall $N\times d$ matrix and $\b$ is an
$N$-dimensional vector. In the context of sketching, a classical way
of computationally reducing this task to a much smaller and more
tractable instance is via the so-called Gaussian embedding: We apply a
randomized linear transformation to both $\A$ and $\b$, where the
transformation is defined by an $n\times
N$ sketching matrix $\S$ with i.i.d.~Gaussian entries from $\Nc(0,1/n)$ to
produce a much smaller $n\times d$ sketched regression task $(\X,\y)=(\S\A,\S\b)$, where
$n\ll N$. Remarkably, for an arbitrary fixed input $(\A,\b)$, the resulting
algorithmically generated data sample $(\X,\y)$ exactly matches the
standard Gaussian random design: each row vector $\x_i^\top$ of $\X$
is a Gaussian
vector with covariance $\Sigmab=\frac1{n}\A^\top\A$, whereas each
entry of $\y$
is distributed according to a linear model $y_i=\x_i^\top\w^*+\xi_i$ with
independent mean zero Gaussian noise.\footnote{This follows because
$\x_i^\top\w^*$ and $\y_i-\x_i^\top\w^*$ are uncorrelated and
jointly Gaussian, and therefore, independent.} Note that we did not impose any
linear noise model on the original problem $(\A,\b)$, but rather, it
arises naturally through the sketching transformation. In particular, this allows us to
derive the exact expected approximation error of 
the sketched ordinary least squares (OLS)~estimator:
\begin{align}
  \E\big[L(\wbh) - L(\w^*)\big] = \frac d{n-d-1}\cdot L(\w^*)
  \qquad\text{for}\qquad\wbh=\argmin_{\w}\|\X\w-\y\|^2.\label{eq:ols-formula}
\end{align}

This statistical model implies many other strong
performance guarantees for a variety of tasks
such as covariance estimation, OLS, Lasso, PCA etc. \citep[e.g.,][]{koltchinskii2017concentration,DW15_TR,miolane2021distribution},
which directly apply to data transformed by Gaussian embeddings, regardless of the
distribution of the matrix $\A$.  We refer
to this phenomenon as Algorithmic Gaussianization. Unfortunately, Gaussian
embeddings carry a substantial preprocessing cost, compared to, say,
uniformly down-sampling the data, and so, many other sketching
techniques have been proposed, e.g., CountSketch, Sparse Johnson-Lindenstrauss
Transforms (SJLT), Subsampled Randomized
Hadamard Transforms (SRHT) and Leverage Score Sampling \citep[e.g.,][]{sarlos-sketching,drineas2006sampling,cw-sparse}, which can be viewed as 
computationally efficient algorithms for transforming a large dataset into a
small data sample. Existing work has shown various approximation
guarantees for these algorithms, via arguments based on subspace
embeddings and the Johnson-Lindenstrauss property \citep{woodruff2014sketching,DM16_CACM,DM21_NoticesAMS}. 
However, even though some level of gaussianization is implied by these
guarantees, exact parallels with Gaussian
random designs are not~available.


In this work, we establish a framework for studying Algorithmic
Gaussianization. Specifically, in our main result
(Theorem~\ref{c:less-embeddings}), we characterize
the total variation distance between a sketched data sample and
the closest sub-gaussian random design. This approach can be used to
show sub-gaussian properties for extremely sparse
sketching transformations, including the so-called Leverage Score
Sparsified (LESS) embeddings, whose time
complexity matches that of other state-of-the-art techniques \citep{less-embeddings}.
For these transformations, our result yields a sub-gaussian property known as the Hanson-Wright inequality
\citep{rudelson2013hanson}, which is widely used to
establish both asymptotic and non-asymptotic guarantees for
statistical models. We also provide a matching lower bound
(Theorem~\ref{t:lower}), which shows that some sketching
methods, such as leverage score sampling, do not enjoy the same
sub-gaussian properties.

We illustrate the strength of our framework on
several examples, including least squares (Theorem~\ref{t:main-app}),
Lasso regression (Corollary~\ref{c:lasso}), and Randomized
SVD (Corollary~\ref{c:rand-svd} in Appendix~\ref{s:examples}). In particular, as a second main result of independent interest, we show that the
Gaussian approximation error formula \eqref{eq:ols-formula} for the
sketched OLS estimator can be non-asymptotically extended to efficient
sparse sketches such as LESS~embeddings. Namely, we show that
  for any $\A$ and $\b$, if $\S$ is a LESS embedding matrix, then the sketched OLS estimator
  satisfies:
  \begin{align*}
  \text{(Theorem \ref{t:main-app})}\qquad  \E\big[L(\wbh) - L(\w^*)\big] \approx_{1+\epsilon} \frac
    d{n-d-1}\cdot L(\w^*) \qquad\text{for}\qquad\epsilon =\tilde O(1/\sqrt d).
  \end{align*}
Here, $a\approx_{1+\epsilon}b$ denotes a
$(1+\epsilon)$-approximation $b/(1+\epsilon)\leq a\leq(1+\epsilon)b$.
This type of guarantee for sketched OLS is a first of its kind
(outside of Gaussian embeddings), as prior work only showed
upper bounds (with additional constant and logarithmic factors)
of the form $\tilde O(d/n)\cdot L(\w^*)$ \citep[e.g., see][]{woodruff2014sketching}, or
asymptotic results \citep{dobriban2019asymptotics}. Remarkably, we
show empirically (Appendix \ref{s:experiments}) that for LESS embeddings the above approximation 
error estimate is extremely accurate, whereas for other popular sketching methods
that are not covered by our theory (such as row sampling and SRHT),
the approximation error can be much more problem~dependent. 


\section{Main result: Hanson-Wright inequality for sparse sketches} 
\label{s:app-less}
In this section, we present our main result, showing that
certain sparse sketching operators produce data sketches whose
distributions are
nearly indistinguishable from sub-gaussian random designs.


As there are a number of closely related notions of sub-gaussianity, we
clarify this here, with a more detailed discussion in
Section~\ref{s:background} (see Figure \ref{f:hierarchy} for an illustration of different
sub-gaussian properties). 
We say that a variable $X$ is $K$-sub-gaussian if its
corresponding sub-gaussian Orlicz norm, i.e., $\|X\|_{\psi_2} = \inf\{t>0:\,\E\,\exp(X^2/t^2)\leq 2\}$, is bounded by $K$.
For a random vector, perhaps the simplest and most popular model of
sub-gaussianity is to assume that it has independent sub-gaussian 
entries. However, entry-wise independence is a fairly strong
assumption. 
 The multivariate sub-gaussian norm, i.e.,
$\|\x\|_{\psi_2}=\sup_{\v:\|\v\|=1}\|\v^\top\x\|_{\psi_2}$, is a popular
relaxation that allows for sub-gaussian vectors with dependent
entries. However, this
notion is too weak for some applications, e.g., in high-dimensional
statistics \citep{bai2010spectral}, or to obtain our least squares result (Theorem \ref{t:main-app}). These settings
often rely on a stronger property,
called the Hanson-Wright inequality,
which has seen significant interest in the literature \cite[e.g.,][]{hsu2012tail,rudelson2013hanson,adamczak2015note,vershynin2020concentration,bamberger2021hanson}.
In the classical version given below \citep[due to][]{rudelson2013hanson}, the inequality is
established for vectors with independent sub-gaussian entries, however, in
our main results, we demonstrate that it is applicable far beyond that
setting. In the following statement, $c$ denotes an absolute constant.
\begin{lemma}[Hanson-Wright inequality]
  \label{l:hanson-wright}
Let $\x$ have independent $K$-sub-gaussian entries with
mean zero and unit variance. Then, it satisfies the Hanson-Wright
inequality with constant $K$: 
\begin{align}
  \Pr\Big\{|\x^\top\B\x-\tr(\B)|\geq t\Big\}\leq
  2\exp\bigg(-c\min\Big\{\frac{t^2}{K^4\|\B\|_F^2},\frac{t}{K^2\|\B\|}\Big\}\bigg)
  \qquad\text{for any \ $\B$}.\label{eq:hanson-wright}
\end{align}
\end{lemma}

Consider an $n\times d$ random matrix $\X$. We will say that $\X$ is a
sub-gaussian random design satisfying Hanson-Wright inequality
with constant $K$, if it can be decomposed so
that $\X = \Z\Sigmab^{1/2}$, where $\Sigmab$ is the positive definite covariance
matrix and $\Z$ consists of i.i.d.~isotropic row vectors $\z_i^\top$
that satisfy~\eqref{eq:hanson-wright}. Note that the rows of $\Z$ may not have independent entries.

Before we present our main result for sparse sketches, let us start
with a simpler example of dense sub-gaussian sketches.  Consider an
$N\times d$ data matrix $\A$ with full column rank and an $n\times N$ sketching matrix $\S$
consisting of independent random $\pm1/\sqrt n$ entries (scaled random
signs). The scaling in $\S$ is standard, and chosen so that the singular values of
the resulting sketch $\S\A$ are of the same order as
the singular values of $\A$. Now, let $\Sigmab = \A^\top\A$ and
$\Z=\sqrt n\S\U$, where $\U=\A\Sigmab^{-1/2}$. 
Then, we have $\S\A =
\frac1{\sqrt n}\Z\Sigmab^{1/2}$, where the scaling by $1/\sqrt n$ could also be absorbed into
$\Sigmab$. It is easy to see that the rows of $\Z$ are i.i.d.~and
isotropic. Moreover, since the rows of $\sqrt n\S$ have
i.i.d.~sub-gaussian entries with $K=O(1)$, and the Hanson-Wright
inequality is preserved under the transformation from $\x$ to $\U^\top\x$, it follows from
Lemma~\ref{l:hanson-wright} that the rows of $\Z$ also satisfy 
\eqref{eq:hanson-wright} with the same constant $K$. 
Even in
this example, the rows of $\Z$ do not have independent entries,
which is why we use inequality \eqref{eq:hanson-wright} to define sub-gaussianity,
rather than the assumptions of Lemma~\ref{l:hanson-wright}.

%

\ifpicture
\begin{figure}
  \centering
      \begin{tikzpicture}[scale=0.5]
      \pgfmathsetmacro{\aa}{0.2};
      \draw (1,0) rectangle (5,2);
      \draw (3,2.3) node {\mbox{\scriptsize $\S$}};
      \draw (3,0.5) node {\mbox{\scriptsize $k=1$}};
      \draw (3,-0.5) node {\mbox{\scriptsize Leverage score sampling}};
      \draw (5.5,1) node {\mbox{$\times$}};
  \draw (7,2.3) node {\mbox{\scriptsize $\A$}};  
  \draw[fill=red!20] (6,-2) rectangle (8,2);
\draw (1,1) rectangle (5,1.2);

 \draw[fill=blue!30] (1+\aa,1) rectangle (1+\aa + 0.2,1.2);
\draw[fill=blue!30] (6, 1.8 - \aa) rectangle (8, 2 - \aa);    

\end{tikzpicture}%
\hspace{9mm}
\begin{tikzpicture}[scale=0.5]

      \pgfmathsetmacro{\aaa}{0};
      \pgfmathsetmacro{\bbb}{0.4};
      \pgfmathsetmacro{\ccc}{1};
      \pgfmathsetmacro{\ddd}{1.2};
      \draw (1,0) rectangle (5,2);
      \draw (3,2.3) node {\mbox{\scriptsize $\S$}};
            \draw (3,0.5) node {\mbox{\scriptsize $k=d$}};
      \draw (3,-0.5) node {\mbox{\scriptsize LESS embedding}};
  \draw (5.5,1) node {\mbox{$\times$}};
  \draw (7,2.3) node {\mbox{\scriptsize $\A$}};  
  \draw[fill=red!20] (6,-2) rectangle (8,2);
\draw (1,1) rectangle (5,1.2);

    \draw[fill=blue!30] (1+\aaa,1) rectangle (1+\aaa + 0.2,1.2);        
    \draw[fill=blue!30] (6, 1.8 - \aaa) rectangle (8, 2 - \aaa);
    \draw[fill=blue!30] (1+\bbb,1) rectangle (1+\bbb + 0.2,1.2);        
    \draw[fill=blue!30] (6, 1.8 - \bbb) rectangle (8, 2 - \bbb);    
    \draw[fill=blue!30] (1+\ccc,1) rectangle (1+\ccc + 0.2,1.2);        
    \draw[fill=blue!30] (6, 1.8 - \ccc) rectangle (8, 2 - \ccc);
    \draw[fill=blue!30] (1+\ddd,1) rectangle (1+\ddd + 0.2,1.2);        
    \draw[fill=blue!30] (6, 1.8 - \ddd) rectangle (8, 2 - \ddd);

  \end{tikzpicture}%
  \hspace{9mm}
  \begin{tikzpicture}[scale=0.5]
      \pgfmathsetmacro{\a}{1.4};
      \pgfmathsetmacro{\b}{2.8};
      \pgfmathsetmacro{\c}{3.2};
      \pgfmathsetmacro{\d}{4};
      \pgfmathsetmacro{\aa}{0.2};
      \pgfmathsetmacro{\aaa}{0};
      \pgfmathsetmacro{\bbb}{0.4};
      \pgfmathsetmacro{\ccc}{1};
      \pgfmathsetmacro{\ddd}{1.2};
      \draw (1,0) rectangle (5,2);
      \draw (3,2.3) node {\mbox{\scriptsize $\S$}};
            \draw (3,0.5) node {\mbox{\scriptsize $k=N$}};
      \draw (3,-0.5) node {\mbox{\scriptsize Sub-gaussian sketch}};

\draw[fill=blue!30] (1,1) rectangle (5,1.2);
  \draw (5.5,1) node {\mbox{$\times$}};
  \draw (7,2.3) node {\mbox{\scriptsize $\A$}};  
  \draw[fill=blue!30] (6,-2) rectangle (8,2);
\draw (1,1) rectangle (5,1.2);

    %




  \draw (9.8,2.3) node {\mbox{\fontsize{7}{7}\selectfont $\ell_i(\A)$}};
  \foreach \i in {5,...,24}
  {
    \pgfmathtruncatemacro{\x}{\i^2};
    \draw[fill=darkgreen!30] (9, 0.2 * \i - 3)
    rectangle (9 + 0.0025 * \x, 0.2 * \i - 2.8);
  }

\end{tikzpicture}%
\caption{Illustration of $p$-sparsified sketches, highlighting the
interaction of one row of $\S$ with $\A$.}\label{f:sketches}
\end{figure}
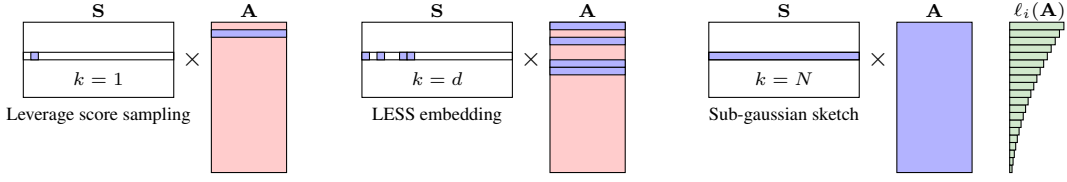
\fi

We next show that a similar reduction can be achieved for a class of
sparse sketching matrices.
\begin{definition}[$p$-sparsified sub-gaussian sketch]\label{d:leverage-score-sparsifier}
Let $p=(p_1,...,p_N)$ be a probability distribution. 
A $p$-sparsified sub-gaussian matrix $\S$ of size $n$ with $k$
non-zeros per row has $n$ i.i.d.\ row vectors
$\frac1{\sqrt{n}}\sum_{i=1}^k\frac{r_i}{\sqrt{kp_{I_i}}}\e_{I_i}^\top$,
where $I_i\sim p$ and $r_i$ are i.i.d.~mean zero, unit variance and $O(1)$-sub-gaussian.
\end{definition}

\noindent
In the following theorem, our main result, we show that many sparse sketches
\ifpicture as defined above \fi
are very close to a sub-gaussian random design that satisfies the Hanson-Wright
inequality with a small constant $K$. The closeness is measured using
the total variation 
distance, denoted by $d_{\tv}$, which allows transferring virtually
any property from the design matrix to the sketch.
The best guarantees
are obtained when the sparsifying distribution $p$ is close to the so-called
leverage score distribution of matrix~$\A$
\citep{drineas2006sampling}: the $i$th leverage score of a rank $d$ matrix
$\A$ is defined as $\ell_i(\A)=\a_i^\top(\A^\top\A)^{-1}\a_i$, where
$\a_i^\top$ is the $i$th row of $\A$.
Note that we have $\sum_{i=1}^N\ell_i(\A)=d$. Also, in the statement, we use $\C\approx_{1+\epsilon}\D$ to denote a
$(1+\epsilon)$-approximation between two positive semidefinite (psd)
matrices $\C$ and $\D$ in terms of the Loewner psd ordering.
\begin{theorem}[Main result]\label{c:less-embeddings}
Consider an $N\times d$ matrix $\A$ with rank $d$, and a
$p$-sparsified sub-gaussian matrix $\S$ of size $n$ with $k$ non-zeros
per row. Define $\mu\!:=\!\max_i\ell_i(\A)/p_i\!\geq\! d$. For any
$\delta > 0$, there is a $d\times d$ psd matrix $\Sigmabt$ and an
  $n\times d$ random matrix $\Z$ with i.i.d.~mean zero isotropic row
  vectors, each satisfying the Hanson-Wright inequality \eqref{eq:hanson-wright}
with $K=O\big(1+\sqrt{\mu\log(\mu n/\delta)/k}\,\big)$, such that: 
\begin{align*}
  d_{\tv}\big(\, 
  \S\A,\, \tfrac1{\sqrt n}\Z\Sigmabt^{1/2}\,\big)\leq
  \delta\qquad\text{and}\qquad
\Sigmabt\approx_{1+\delta}\A^\top\A.
\end{align*}
\end{theorem}

\begin{figure}
\centering
\includegraphics[width=0.57\textwidth]{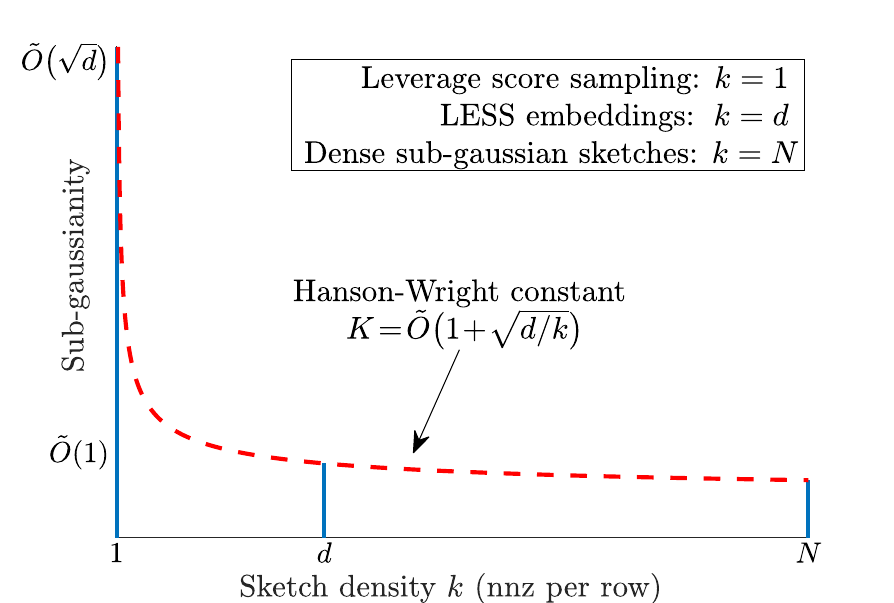}
\vspace{-3mm}
\caption{Illustration for
  Theorem~\ref{c:less-embeddings}.
  Matching lower bound is given in
  Theorem~\ref{t:lower}. Note that a $p$-sparsified sketch with $k=N$
  is not strictly equivalent to a standard sub-gaussian sketch, due to with-replacement
sampling of non-zero entries (chosen to simplify the analysis).}\label{f:plot}
\end{figure}

To illustrate this result, consider the scenario where $N\gg d$ and $p$ is the
approximate leverage score sampling distribution, i.e., $p_i
\approx_{O(1)} \ell_i(\A)/d$ for all $i$, in which case $\mu=\Theta(d)$
and so the Hanson-Wright constant satisfies $K=\tilde
O(1+\sqrt{d/k})$. In this case, as shown in Figure
\ref{f:sketches} (where the rows of $\A$ are sorted by leverage scores for the
purpose of illustration), as we vary sketch density~$k$, i.e.,
the number of non-zeros per row of $\S$, we recover:
leverage score 
sampling for $k=1$, LESS embeddings for $k=d$, and dense sub-gaussian
sketches for $k=N$ (up to minor differences in the definitions, which do not affect the
conclusions). For leverage score sampling, we can
only get $K=\tilde O(\sqrt d)$, which is essentially vacuous. However,
to get the optimal $K=\tilde O(1)$, we only need $k=d \ll
N$. Increasing $k$ further has little effect on the
Hanson-Wright constant, as we observe in Figure \ref{f:plot}, which
illustrates the dependence of $K$ on sketch density. In Theorem~\ref{t:lower} below, we provide a
nearly-matching lower bound of $K=\tilde\Omega(1+\sqrt{d/k})$, which in particular
confirms that leverage score sampling does not enjoy a near-optimal
Hanson-Wright guarantee.

\begin{remark}\label{r:less}
Below are the key examples of sparse sketches that fall under
Theorem~\ref{c:less-embeddings}, for which
the Hanson-Wright constant satisfies $K=\tilde O(1)$:
\begin{enumerate}
  \item \emph{\textbf{LESS embeddings.}} Here, $p$ is
    an approximate leverage score sampling distribution, i.e.,
    we have $p_i \approx_{O(1)} \ell_i(\A)/d$ for all $i$. The number of
    non-zeros needs to satisfy $k=\Omega(d)$. Sketching takes
    $O(\nnz(\A)\log N \!+\! nd^2)$ time, where $\nnz(\A)$ is the number of
    non-zero entries~in~$\A$. 
  \item \emph{\textbf{Uniformly sparsified sketches.}} We use a uniform
    distribution $p$, i.e., $p_i=\frac1N$ for all $i$. The number of
    non-zeros needs to satisfy $k=\Omega(N\cdot\max_i\ell_i(\A))$, which
    can be anywhere between $d$ and $N$, depending on $\A$. Here, sketching
    takes $O(ndk)$ time.  
  \item \emph{\textbf{Preconditioned sparse sketches.}} Replacing $\A$
    with $\tilde\A=\H\D\A$, where $\H$ is a fast Hadamard transform
    and $\D$ is diagonal with random $\pm1$ entries, we can ensure that all 
    leverage scores are nearly uniform:
    $\ell_i(\tilde\A)\approx d/N$  \citep{tropp2011improved}, while
    preserving the data covariance, $\tilde\A^\top\tilde\A=\A^\top\A$. Then,
    the uniformly sparsified sketch 
needs only $k=\Omega(d)$ non-zeros. This procedure takes $O(Nd\log N +
nd^2)$ time for preconditioning and sketching. 
\end{enumerate}
\end{remark}
Based on the above result, LESS embeddings and preconditioned
sparse sketches provide the most effective
Gaussianizing property for arbitrary data matrices $\A$. However,
simpler uniformly sparsified sketches can be equally effective for
certain input matrices (depending on the leverage score distribution).
We note that the nearly-linear time complexity of LESS embeddings is
primarily based on the cost of approximating the leverage scores of
$\A$ \citep{fast-leverage-scores}, which can be
done in $O(\nnz(\A)\log N + d^3\log(d))$ time. Thus, when using LESS embeddings, our
Hanson-Wright reduction can be achieved in 
time nearly linear in the input size (a.k.a.~input sparsity time),
with its complexity matching state-of-the-art sketching methods up
to logarithmic factors (see Section \ref{s:related-work}). 

Certain sparse sketches, such as
CountSketch or SJLT, do not
technically fit into our framework (we assume row-wise
independence, while they have column-wise
independence). Yet, these methods align closely with uniformly
sparsified sketches (e.g., for CountSketch, with $k\approx
N/n$ non-zeros per row). This suggests 
that the sparsity needed to recover sub-gaussian
behavior with these methods depends both on the leverage score
distribution and the dimensions of the data matrix.

Next, we demonstrate the sharpness of 
Theorem~\ref{c:less-embeddings} by constructing examples of data matrices and
$p$-sparsified sketches for which one cannot find a sufficiently good
sub-gaussian random design that is close in total variation
distance. Importantly, the result gives a lower bound of
$K\!=\!\tilde\Omega(\sqrt d)$ for leverage score sampling (i.e., $k=1$), and shows
that LESS embeddings require $k\!=\!\tilde\Omega(d)$ non-zeros per row to
get $K=\tilde O(1)$, matching our
upper bounds up to logarithmic factors (see Figure~\ref{f:plot}).

\begin{theorem}[Lower bound]\label{t:lower}
For any $N\geq d$, there is a rank $d$ 
matrix $\A\in\R^{N\times d}$ with the property:\\
Given any
$n,k\geq 1$, let $\S$ be a
$p$-sparsified sub-gaussian matrix  of size $n$ with $k$ non-zeros
per row, where variables $r_i$ are random $\pm1$ signs, and the
distribution $p$ satisfies $p_i\approx_{O(1)} \ell_i(\A)/d$ for all~$i$.
Suppose there is $\Z$ having i.i.d.~isotropic row vectors satisfying
Hanson-Wright inequality \eqref{eq:hanson-wright} with 
$K$, such that $d_{\tv}(\S\A,\frac1{\sqrt n}\Z\Sigmabt^{1/2})\leq 1/2$\,
for some $\Sigmabt\approx_{O(1)}\A^\top\A$. Then:
\begin{align*}
  K = \Omega\Big(1+ \sqrt{\tfrac{d/k}{\log d}}\,\Big).
\end{align*}
\end{theorem}
\begin{remark}
  Recall that when $p_i\approx_{O(1)}\ell_i(\A)/d$ (i.e., approximate
  leverage score sampling), then $\mu=\Theta(d)$ in
  Theorem~\ref{c:less-embeddings}, so the lower bound matches our
  upper bound up to logarithmic factors for leverage score sampling
  and LESS embeddings. In Appendix~\ref{a:lower} we show an
  even stronger result, lower bounding the 
  sub-gaussian norm $\|\cdot\|_{\psi_2}$ instead of the Hanson-Wright
  constant $K$, so that $\Z$ cannot be a sub-gaussian design even
  in this weaker sense (see Figure~\ref{f:hierarchy} for comparison).
\end{remark}
To demonstrate the benefits of our main results, recall that
small total variation distance, namely $d_{\tv}(\S\A,\frac1{\sqrt n}\Z\Sigmabt^{1/2})\leq\delta$, means that we can couple the
sketch $\S\A$  with the sub-gaussian design so that they are
identical with probability 
$1-\delta$. This effectively allows us to transfer any property that 
holds with high probability for the sub-gaussian design
to a property that holds for the sketch. The only
caveat is that the covariance matrix gets slightly
distorted in the process, but both this distortion and the failure
probability are controlled by $\delta$ which can be made negligibly
small. We use this in Section~\ref{s:apps} to show new
guarantees for sparse sketches in least squares and
Lasso regression.



\section{Applications: Least squares and Lasso regression}
\label{s:apps}
We next present some of the implications of Theorem \ref{c:less-embeddings}. In
Section \ref{s:main-app}, we present our main application, giving a
precise characterization of the expected approximation error of the sketched
least squares estimator with LESS embeddings. Here, this is not
merely a corollary of our main result, but rather a result of
independent interest, and to our knowledge, it was not previously
known even for dense sub-gaussian sketches. Then, in
Section~\ref{s:lasso}, we demonstrate how our results imply new efficient
algorithms for approximately solving Lasso regression and other
constrained optimization tasks. Further applications to low-rank
approximation, as well as corollaries for the subspace embedding and
Johnson-Lindenstrauss properties of LESS embeddings, are given in Appendix~\ref{s:examples}.

\subsection{Sketched least squares with LESS embeddings}
\label{s:main-app}
As a motivating application of our theory, we study the approximation
properties of the sketched least
squares estimator, showing that, when the sketch is a LESS embedding,
then it achieves nearly the same expected approximation error as for a
Gaussian embedding, even down to constant factors.

Consider an $N\times d$ data matrix $\A$ with full column rank, where $N\gg d$, and an
$N$-dimensional vector $\b$. Recall that our goal is to approximately
solve the least squares task $(\A,\b)$,
i.e., find:
\begin{align*}
  \w^* = \argmin_{\w} L(\w) =
  (\A^\top\A)^{-1}\A^\top\b,\quad\text{where}\quad L(\w) = \|\A\w-\b\|^2.
\end{align*}
Given an $n\times N$ sketching matrix
$\S$,  the sketched least squares estimator $\wbh$ for this model is the
solution of the sketched version of the problem, i.e., $(\S\A,\S\b)$.
If $\S$ is a Gaussian embedding, then via a reduction to a Gaussian
random design, the expected error of this estimator is given by the
exact formula \eqref{eq:ols-formula}, but the cost of constructing such a
sketch is $O(Ndn)$, which is prohibitively expensive. If we let $\S$ be a LESS embedding matrix with
$d$ non-zeros per row (see Remark \ref{r:less}), then
the cost of computing this estimator is $O(\nnz(\A)\log N + nd^2)$,
which includes the cost of approximating the leverage scores, sketching
the data, and computing $\wbh$ from the sketch. We show that the
resulting estimator enjoys nearly the same expected error guarantee as the
Gaussian embedding. To avoid numerical precision issues which may
occur when computing the exact expectation (there is a small but
non-zero probability that $\S\A$ is very ill-conditioned), we assume
that: the LESS embedding matrix is constructed so that the random
variables $r_i$ from Definition \ref{d:leverage-score-sparsifier} are
random $\pm1$ sign variables, and the probability distribution satisfies
$p_i\approx_{O(1)}\ell_i(\A)/d+1/N$. We also restrict 
the entries of the sketched solution $\wbh$ to a finite
range $\mathcal R = [-D,D]$, for a sufficiently large
$D=\poly(N,\kappa(\A),\|\A\|,\|\b\|)$ which is absorbed by the
logarithmic factors in the result (see Appendix \ref{a:main-app} for details). 
\begin{theorem}\label{t:main-app}
  Fix $\A$ of rank~$d$ and let $\S$
be a LESS embedding as above, with $n\geq \tilde O(d)$ rows and $k=d$
non-zeros per row.
For any $\b$, the estimator
$\wbh=\argmin_{\w\in\mathcal R^d}\|\S(\A\w-\b)\|^2$ satisfies:
\begin{align*}
  \E\big[L(\wbh) - L(\w^*)\big]\,\approx_{1+\epsilon}
  \frac{d}{n-d-1}\cdot L(\w^*)\qquad\text{for}\qquad\epsilon =
  \tilde O(1/\sqrt d\,),
\end{align*}
where $L(\w)=\|\A\w-\b\|^2$ and $\w^*=\argmin_\w L(\w)$.
\end{theorem}
The result easily extends to any sketching method covered by
Remark \ref{r:less}, as long as the entries of $\S$ are almost surely
bounded, as well as to dense random $\pm1$ matrices, affecting
only the polylogarithmic factors hidden in $\tilde O$.
\begin{remark}
  The obtained expression for the expected multiplicative
approximation error, i.e.,
$\frac{d}{n-d-1}$, is completely independent of the data matrix $\A$
or the regression vector $\b$, and it matches the formula
\eqref{eq:ols-formula} for the Gaussian embedding.
Comparable non-asymptotic guarantees for other fast sketching methods, such as SJLT, SRHT,
or leverage score sampling, take the form of $\tilde O(d/n)\cdot
L(\w^*)$, where the precise constant and logarithmic terms are problem-dependent
\citep[e.g., see][]{woodruff2014sketching,chen2017condition}. To our knowledge, this result
was not previously known even for dense sub-gaussian sketches.
\end{remark}

\paragraph{Proof sketch}
The proof, given in Appendix~\ref{a:main-app}, proceeds differently from
the existing bounds for sketched least squares. Our approach is
inspired by the asymptotic analysis of the OLS estimator in
high-dimensional statistics \cite[e.g.,][]{HMRT19_TR}, but we
obtain a non-asymptotic bound and we avoid any statistical
assumptions on the label noise. We start by rewriting 
the expected loss $\E[L(\wbh)]$ as the expectation of the
leave-one-out cross-validation estimator,
$L_{\CV}=\sum_{i=1}^n(\tilde\a_i^\top\wbh^{-i}-\tilde b_i)^2$, where
$(\tilde\a_i^\top,\tilde b_i)$ is the $i$th row of the sketched problem
$(\S\A,\S\b)$ and $\wbh^{-i}$ is computed by
excluding this row from the sketch. We then use the Hanson-Wright inequality via
Theorem~\ref{c:less-embeddings} to show that all of the leverage  
scores of the sketch are with high probability very close to
uniform (Lemma~\ref{l:event}):
\begin{align*}
  \ell_i(\S\A) \approx_{1+\epsilon}\frac dn\quad\text{ with
  }\quad\epsilon=\tilde O(1/\sqrt d\,)\quad\text{ for all $i$.}
\end{align*}
Using a standard short-cut formula, this allows us to approximate
the expectation of $L_{\CV}$ as follows:
\begin{align*}
  \E[L_{\CV}] \approx_{1+\epsilon'}
  \big(1-\tfrac dn\big)^{-2}\E\big[\|\S(\A\wbh-\b)\|^2\big]\quad\text{with
  }\quad\epsilon'=\tilde O(\sqrt d/n).
\end{align*}
Finally, our key contribution in the analysis is a new approximate expectation
formula for the sketch of the so-called \emph{hat matrix},
$\H=\A(\A^\top\A)^{-1}\A$ (i.e., the orthogonal projection onto the
column-span of $\A$), which arises when deriving the optimum least
squares loss: $L(\w^*)=\b^\top(\I-\H)\b$.
Specifically, in Lemma \ref{l:hatmatrix} we
show that the hat
matrix $\Hbh$ of the sketch $\S\A$ satisfies:
\begin{align*}
   \E\big[\,\S^\top(\I-\Hbh)\S\,\big] \approx_{1+\epsilon'} \big(1-\tfrac
    dn\big)\cdot (\I-\H)\quad\text{with}\quad \epsilon' = \tilde O(\sqrt d/n),
\end{align*}
where $\approx_{1+\epsilon}$ is defined in terms of the psd matrix
ordering, and the formula holds exactly when $\S$ is a Gaussian
embedding. The result then follows because $\E[\|\S(\A\wbh-\b)\|^2] =\b^\top\E[\S^\top(\I-\Hbh)\S]\b$.

\paragraph{Experiments}
In Appendix \ref{s:experiments} we empirically show that our estimate
for the expected approximation error of sketched OLS with a LESS
embedding is even more accurate than suggested by the theory, on a
range of benchmark regression tasks. Our experiments also show that
this problem-independent expected approximation error is not always shared
by other fast sketching methods that are considered to have strong
guarantees for least squares, such as SRHTs.  


\subsection{Lasso and constrained least squares}
\label{s:lasso}

In this section, we show how our results can be used to extend
existing sketching guarantees in
constrained least squares optimization 
from dense sub-gaussian sketches to LESS embeddings. This general problem
setting includes such standard tasks as Lasso regression, sparse
recovery and training a support vector machine.

We consider the following constrained optimization task for
$\A\in\R^{N\times d}$ and $\b\in\R^N$:
\begin{align*}
  \text{find}\qquad \w^* =
  \argmin_{\w\in\Cc}L(\w),\qquad\text{for}\qquad L(\w)=\|\A\w-\b\|^2, 
\end{align*}
where $\Cc$ is some convex subset of $\R^d$. Our goal is to
approximately solve this task by solving a sketched version of the~problem:
  $\wbh = \argmin_{\w\in\Cc}\|\S(\A\w-\b)\|^2$,
for an $n\times N$ sketching matrix $\S$. When the sketching matrix is
a sub-gaussian embedding, then \cite{pilanci2015randomized} showed
that the sketch size $n$ needed for $\wbh$ to achieve a
$(1+\epsilon)$-approximation, i.e., $L(\wbh)\leq (1+\epsilon)\cdot
L(\w^*)$, is controlled by the so-called Gaussian width, which can be
viewed as a measure of the degrees of freedom for the constrained
optimization task (in particular, it can be much smaller than the
actual dimension $d$), and is defined as follows:
\begin{align*}
\mathbb W := \mathbb W(\A\mathcal K)=\E_{\g}\Big[\sup_{\u\in\A\Kc\cap \Sc^{N-1}}|\g^\top\u|\Big],
\end{align*}
where $\g$ has i.i.d.~standard Gaussian entries, $\Sc^{N-1}$ is the unit sphere in $\R^N$, and $\Kc$ is the tangent cone of the constraint set $\Cc$ at the
optimum $\w^*$, i.e., the closed convex hull of
$\{\Delta\in\R^d\mid\Delta(\w-\w^*)\text{ for }t\geq 0,\w\in\mathcal
C\}$. Specifically, they showed that if $\S$ consists of
i.i.d. mean zero isotropic sub-gaussian rows and
$n\geq O(\mathbb W^2 / \epsilon^2)$, then with high probability
$\wbh$ is a $(1+\epsilon)$-approximation of
$\w^*$. We use Theorem \ref{c:less-embeddings} to extend this result to LESS
embeddings (see Appendix \ref{a:constrained}).
\begin{theorem}\label{t:constrained}
If $\S$ is a LESS embedding matrix for $\A$, with size $n\geq \tilde O(\mathbb
W^2/\epsilon^2)$, then with high probability $\wbh$ is a
$(1+\epsilon)$-approximation of $\w^*$.
\end{theorem}
We illustrate this claim with
a corollary for Lasso regression, where the constraint set is an
$\ell_1$ ball, i.e., $\Cc=\{\w:\|\w\|_1\leq R\}$. Crucially, here the
sketch size needed to obtain a good approximation scales
with the sparsity of the solution $\w^*$, rather than with the
dimension of the data.
First, we
formulate an $\ell_1$-restricted condition number of $\A$, which is
always smaller than the actual condition number, and which naturally arises in
bounding the Gaussian width:
\begin{align*}
  \kappa = \max_{i}\|\A_{:,i}\|/\gamma_s^-(\A),\qquad\text{where}\qquad
  \gamma_s^-(\A)=\min_{\|\v\|=1,\,\|\v\|_1\leq2\sqrt s}\|\A\v\|.
\end{align*}
\vspace{-4mm}
\begin{corollary}\label{c:lasso}
  Suppose that the Lasso solution $\w^*$ is $s$-sparse. Then, using a LESS
  embedding of size $n\geq \tilde O(\kappa^2s/\epsilon^2)$, we can
  obtain a $(1+\epsilon)$-approximation of $\w^*$ in time
  $O(\nnz(\A)\log N + nd^2)$.
\end{corollary}
Finally, we note that similar results for approximately solving
Lasso and other constrained least squares problems can be obtained for SRHT sketches.
However, there are two key differences: (1) the bounds
obtained for SRHT are weaker than the ones achieved for
LESS or sub-gaussian sketches, e.g., in addition to the Gaussian
width, they also scale with other
problem-dependent quantities such as the Rademacher width
\cite[see Theorem~2 of][]{pilanci2015randomized}; and (2) the cost of
 SRHT is $O(Nd\log N)$, which is higher than the cost of LESS when $\nnz(\A)\ll Nd$.

\section{Key technical result: A Hanson-Wright limit theorem}
\label{s:clt} 
In this section, we present the key technical result that we use to 
prove Theorem \ref{c:less-embeddings}. This result shows that, by relying
on the central limit phenomenon, we can establish that a sum of
bounded random vectors gets very close, in terms of the total
variation distance $d_{\tv}$, to a random vector that
satisfies the Hanson-Wright inequality \eqref{eq:hanson-wright} with a
small constant $K$, even though the sum may not
satisfy this property directly. There has been extensive literature dedicated to extending and
adapting the Hanson-Wright inequality
\cite[e.g.,][]{hsu2012tail,adamczak2015note,vershynin2020concentration,bamberger2021hanson},
and our result is of independent interest in the context of this line
of works. 
\begin{theorem}\label{t:main}
Let $\x_1,\x_2,...$ be i.i.d.~random vectors, where
$\E[\x_i\x_i^\top]=\Sigmab$ is full-rank and
    $\|\Sigmab^{-1/2}\x_i\|_2$ is $M$-sub-gaussian. Define a $k$-gaussianized sample
$\xbt = \frac1{\sqrt k}\sum_{i=1}^k r_i\x_i$,
where $r_i$ are independent mean zero unit variance $R$-sub-gaussian random
variables. For any $\delta > 0$, there is $L =
O\big(1+M\log(M/\delta)/ \sqrt{k}\big)$, a $d\times d$ psd matrix $\Sigmabt$,  and a mean zero isotropic vector
$\z$ satisfying the Hanson-Wright inequality \eqref{eq:hanson-wright} with $K=LR$, such that: 
\begin{align*}
d_{\tv}\big(\, \xbt,\, \Sigmabt^{1/2}\z\,\big)\leq
  \delta\qquad\text{and}\qquad
\Sigmabt\approx_{1+\delta} \Sigmab.
\end{align*}
\end{theorem}
\begin{remark}\label{r:bounded}
Note that $M^2\geq\E[\|\Sigmab^{-1/2}\x_i\|_2^2]=d$,  where $d$ is the dimension~of~$\x_i$.
  If we make a stronger assumption that $\|\Sigmab^{-1/2}\x_i\|_2\leq
  M$ almost surely, then we can obtain a slightly sharper result:
  $L=O(1+M\sqrt{\log(M/\delta)/k})$, which is what we use to establish
  Theorem \ref{c:less-embeddings}. Either of the boundedness assumptions includes
  all distributions over finite populations in general position.
\end{remark}

The proof, which can be found in Appendix~\ref{s:main}, is briefly
summarized here. We start by
defining a high-probability event $\Ec$ ensuring that the sample covariance
matrix $\frac1k\sum_i\x_i\x_i^\top$ is sufficiently well bounded. In the
process, we establish a bound on the matrix moments of the random
covariates $\x_i\x_i^\top$ (Lemma \ref{l:sub-exponential}). Then, we
construct the vector $\z$ so that it is coupled with $\xbt$ 
when the event $\Ec$ holds. This requires introducing a correction
term designed to minimize 
the distortion of the covariance matrix of $\z$, so as to ensure that $\Sigmabt\approx_{1+\delta}\Sigmab$. We then establish the Hanson-Wright
inequality for $\z$ in three steps: First, analyzing the error coming
from the $r_i$'s; then accounting for the randomness in $\x_i$'s; and
finally, bounding the noise coming from the correction term.

At a high level, Theorem \ref{t:main} states that as $k$ goes to infinity, the total
variation distance of $\xbt$ from a certain family of sub-gaussian distributions
goes to zero. Using the total variation distance to
measure central limit behavior (instead of, say, the Wasserstein
distance)  is unusual because of how strong of a guarantee it requires (see
Section~\ref{s:related-work} for discussion). This is possible here because 
we are comparing to a sub-gaussian distribution, instead of directly to a Gaussian. 

Note that the lower bound we provide in Theorem \ref{t:lower} for
$p$-sparsified sketches also constitutes a lower bound for
Theorem \ref{t:main}. In particular, Theorem \ref{t:lower} implies that there are
distributions of $\x_i$ such that $M=\Theta(\sqrt d)$, for which
$K=\tilde\Omega(1+M/\sqrt k)$.
This, up to logarithmic factors, 
matches our upper bound of $K=\tilde O(1+M/\sqrt k)$ samples, and can
be easily extended to $M\gg \sqrt d$.

\paragraph{Proof of Theorem \ref{c:less-embeddings}}
We now briefly explain how Theorem \ref{c:less-embeddings} follows
from Theorem~\ref{t:main}.
Consider an $N\times d$ matrix $\A$ with rank $d$, and let
$\Sigmab=\A^\top\A$. Suppose that the random vectors $\x_i =
\frac1{\sqrt{p_{I_i}}}\a_i^\top$ are defined as scaled row samples from
$\A$, drawn using the index distribution $I_i\sim p$ and scaled so that
$\Sigmab:=\E[\x_i\x_i^\top]=\A^\top\A$. Now, let us construct an
$n\times d$ random matrix $\Xbt$
consisting of rows $\xbt_1,...,\xbt_n$ which are independent
$k$-gaussianized samples defined as in Theorem~\ref{t:main}. This
sketch can be obtained by applying a $p$-sparsified sketching
matrix $\sqrt n\S$ to~$\A$, where each row of $\sqrt n\S$ is given by
$\sum_{i=1}^k\frac{r_i}{\sqrt{kp_{I_i}}}\e_{I_i}^\top$. 
Notice that $\|\Sigmab^{-1/2}\x_i\|_2\leq
\sqrt{\ell_{I_i}(\A)/p_{I_i}}\leq \sqrt\mu$ (as defined in Theorem~\ref{c:less-embeddings}), so we can let the constant $M$ in
Theorem~\ref{t:main} be $\sqrt\mu$, where recall that $\ell_i(\A)$ is the
$i$th leverage score of matrix~$\A$, i.e.,
$\ell_i(\A)
=\a_i^\top(\A^\top\A)^{-1}\a_i=\|\Sigmab^{-1/2}\a_i\|^2$. Note that we
can rely here on the stronger boundedness condition 
as in Remark~\ref{r:bounded}, so that we obtain 
$d_{\tv}(\xbt_i,\Sigmabt^{1/2}\z)\leq \delta$ for each $i$, with
$K=O(1+\sqrt{\mu\log(\mu/\delta)/k})$. Recall that small total
variation distance implies that we can couple $\xbt_i$ with a
corresponding $\Sigmabt^{1/2}\z_i$, where $\z_i$ is an independent
sample of $\z$, so that they are identical with probability $1-\delta$. It
remains to apply a union bound
over the $n$ random vectors $\xbt_1,...,\xbt_n$ and their
corresponding coupled samples $\z_1,...,\z_n$, to show the total variation distance of $n\delta$ between the
entire sketch $\sqrt n\S\A$ and the corresponding sub-gaussian design
$\Z\Sigmabt^{1/2}$. Replacing $\delta$ with $\delta/n$, we obtain the
claim of Theorem~\ref{c:less-embeddings}. 

\section{Background and related work}
\label{s:background}
In this section, we put our results in context by discussing
different types of sub-gaussianity that have been considered in the
literature. Here, we rely on the hierarchy of sub-gaussian
concentration, shown in Figure~\ref{f:hierarchy}, which is 
based on a range of classical results in high-dimensional probability 
\citep[e.g., see][]{vershynin2018high}.
We then discuss how these concepts arise in related
work, particularly in the context of randomized sketching.

\subsection{Sub-gaussian concentration hierarchy for random vectors}
\label{s:hierarchy}
While the notion of sub-gaussian concentration for scalar random
variables can be naturally represented by the Orlicz norm
$\|\cdot\|_{\psi_2}$, the landscape of sub-gaussian concentration
becomes more complex when we consider multivariate distributions. As
discussed in Section~\ref{s:app-less}, the natural multivariate extension
of the Orlicz norm, i.e., the sub-gaussian norm
$\|\x\|_{\psi_2}=\sup_{\v:\|\v\|=1}\|\v^\top\x\|_{\psi_2}$, while
useful for some applications (such as Theorem~\ref{t:constrained}), does
not suffice for other Gaussian-like guarantees (such as Theorem
\ref{t:main-app}).
%
%
A more general approach to quantifying sub-gaussianity of random
vectors is to analyze the concentration of real-valued functions of
these vectors. In this context, let $\Fc$ be some subset of functions
$f:\R^d\rightarrow \R$ with a bounded Lipschitz constant defined as:
$\|f\|_{\Lip}:=\sup_{\u,\v}|f(\u)-f(\v)|/\|\u-\v\|$.
\begin{definition}\label{d:concentration}
  We say that a $d$-dimensional random
vector $\x$ has the concentration property over $\Fc$ with constant $K$
if $\|f(\x)-\E\,f(\x)\|_{\psi_2}\leq K\cdot\|f\|_{\Lip}$ for all
$f\in\Fc$.
\end{definition}
For the sake of simplicity, we focus on
random vectors that are mean zero, $\E[\x]\!=\!\zero$, and
isotropic, i.e., $\E[\x\x^\top]\!=\!\I$. Gaussian vectors (as well as,
e.g., random vectors on the sphere) satisfy the concentration
property over all Lipschitz functions, with an absolute constant $K$.
However, as this condition is very restrictive, smaller function
families have been considered. If we restrict $\Fc$ to
convex functions, then vectors with independent bounded entries
satisfy the concentration property (but vectors with i.i.d. sub-gaussian entries
may not). On the other end of the concentration hierarchy, illustrated in
Figure~\ref{f:hierarchy}, is the family of
linear functions, i.e., $\Fc=\{f:f(\x)\!=\!|\v^\top\x|,\text{ for some
}\v\!\in\!\R^d\}$, for which the concentration constant of $\x$ becomes
simply its sub-gaussian norm. Of primary interest to this work is
the family of Euclidean functions \citep{vershynin2020concentration},
defined as $\Fc=\{f:f(\x)=\sqrt{\x^\top\B\x}\
  \text{ for psd }\,\B\}$,
which falls between convex and linear in the hierarchy, and can be
used to characterize vectors with i.i.d.~sub-gaussian entries.
\begin{proposition}\cite[Section 6.3]{vershynin2018high}\label{p:euclidean}
  If an isotropic mean zero vector satisfies the
  Hanson-Wright inequality \eqref{eq:hanson-wright},
  then it has $O(K)$ sub-gaussian norm and  $O(K^2)$-Euclidean concentration.
\end{proposition}

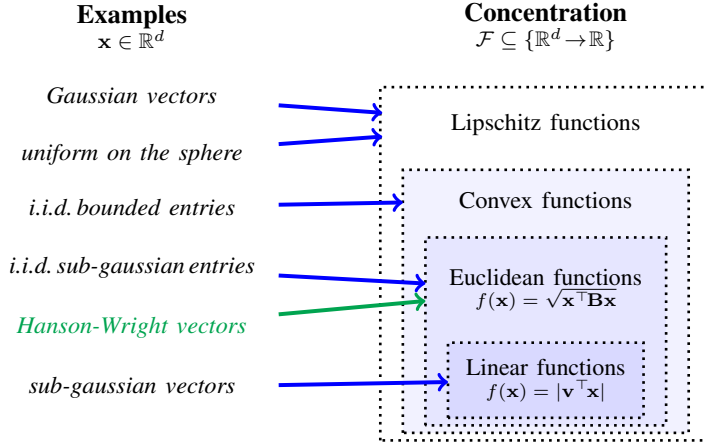
\begin{figure} 
\centering  \begin{tikzpicture}
    \pgfmathsetmacro{\a}{5.5};

    \draw (-\a,0.3) node () {\parbox{4cm}{\centering\textbf{Examples}\\
        \footnotesize $\x \in \R^d$}};
    \node [text width=3.65cm,text centered]
    (gaussian) at (-\a,-.6)
    {\small \it Gaussian vectors};
    \node [text width=3.65cm,text centered]
    (sphere) at (-\a,-1.4)
    {\small\it uniform on the sphere};
    \node [text width=3.65cm,text centered]
    (bounded) at  (-\a,-2.1)
    {\small\it i.i.d.\,bounded entries};
    \node [text width=3.65cm,text centered]
    (subg) at (-\a,-2.9)
    {\small\it i.i.d.\,sub-gaussian\,entries};
    \node [text width=3.65cm,text centered]
    (gaussianized) at  (-\a,-3.7)
    {\small\it \textcolor{black!1!green}{Hanson-Wright vectors}};
    \node [text width=3.65cm,text centered]
    (sub-gaussian) at  (-\a,-4.5)
    {\small\it sub-gaussian vectors};
    
    \draw (0,.3) node () {\parbox{4cm}{\centering\textbf{Concentration}\\
      \footnotesize $\Fc\subseteq \{\R^d\!\rightarrow \!\R\}$}};

    \draw[dotted,line width=1] (-2.2,-5.2) rectangle (2.2,-0.5);
  \node[text width=3.5cm,text centered] (lipschitz) at (0,-1) {\small Lipschitz functions}; 
  \draw[dotted,line width=1,fill=blue!5] (-1.9,-5.1) rectangle (1.9,-1.6);
  \node[text width=3.15cm,text centered] (convex) at (0,-2) {\small Convex functions};     
  \draw[dotted,line width=1,fill=blue!10] (-1.6,-5) rectangle (1.6,-2.5);
 \node[text width=3cm,text centered] (norm) at (0,-3.2) {\small
   \parbox{3cm}{\centering Euclidean functions\\\scriptsize $f(\x)=\sqrt{\x^\top\!\B\x}$}};  
 \draw[dotted,line width=1,fill=blue!15] (-1.3,-4.9) rectangle (1.3,-3.9);
 \node[text width=3cm,text centered] (linear) at (0,-4.4) {\small
   \parbox{3cm}{\centering Linear functions\\\scriptsize
     $f(\x)=|\v^\top\!\x|$}};    
            
    \draw [->,blue,line width=1.5,shorten >=3mm] (gaussian) -- (lipschitz);
    \draw [->,blue,line width=1.5,shorten >=3mm] (sphere) -- (lipschitz);
    \draw [->,blue,line width=1.5,shorten >=2mm] (bounded) -- (convex);
    \draw [->,black!1!green,line width=1.5,shorten >=-0.25mm] (gaussianized) -- (norm);
    \draw [->,blue,line width=1.5,shorten >=-3.25mm] (sub-gaussian) -- (linear);
    \draw [->,blue,line width=1.5,shorten >=-0.25mm] (subg) -- (norm);


  \end{tikzpicture}
  \caption{Hierarchy of sub-gaussian concentration for mean zero
    isotropic random vectors.
    Right column represents
    different function classes $\Fc$ used in Definition
    \ref{d:concentration}, and left column provides examples of
    random vectors that satisfy the property with a particular
    function class (the larger the function class, the stronger the
    concentration property). ``Hanson-Wright
    vectors'' refers to vectors satisfying the Hanson-Wright
    inequality \eqref{eq:hanson-wright}, whereas ``sub-gaussian
    vectors'' have bounded multivariate sub-gaussian
    norm~$\|\cdot\|_{\psi_2}$. }
  \label{f:hierarchy}
\end{figure}

 \subsection{Related work}
\label{s:related-work}

Randomized sketching has emerged out of theoretical computer science,
as part of the broader area of Randomized Numerical Linear Algebra
\citep[RandNLA;][]{woodruff2014sketching,DM16_CACM,DM21_NoticesAMS}, including
computationally efficient techniques such as the Subsampled 
Randomized Hadamard Transform \citep[SRHT;][]{ailon2009fast}, Leverage Score
Sampling \citep{drineas2006sampling}, the CountSketch \citep{cw-sparse}, Sparse
Johnson-Lindenstrauss Transforms \citep[SJLT;][]{nn-sparse},
Leverage Score Sparsified
embeddings \citep[LESS;][]{less-embeddings} and
Determinantal Point Processes \citep[DPPs;][]{dpp-intermediate}.
More recently, there has been an increased
interest in the statistical analysis of sketching and importance sampling
techniques, for example in the context of linear regression
\citep{ping-ma2014}, kernel ridge regression \citep{ridge-leverage-scores},
\ifisarxiv
experimental design \citep{minimax-experimental-design,derezinski2020bayesian},
\fi
model
averaging \citep{sketched-ridge-regression}, and the bootstrap
\citep{lopes2019bootstrapping}. Some works have shown deep connections between the
performance of certain sketching methods and Gaussian/sub-gaussian
embeddings: for DPPs, e.g., in the context of low-rank approximation
and stochastic optimization
\ifisarxiv
\citep{nystrom-multiple-descent,precise-expressions,randomized-newton};
\else
\citep{precise-expressions};
\fi
 for LESS embeddings in the context of
inverse covariance estimation and stochastic
optimization \citep{newton-less,less-embeddings}; and for
SRHT, a similar phenomenon has been studied asymptotically in linear
regression \citep{dobriban2019asymptotics} and stochastic optimization
\citep{lacotte2020optimal}. Despite this 
extensive literature, no general non-asymptotic equivalence result was known
between any of these fast sketching methods and sub-gaussian
embeddings. Our results show such equivalence for LESS embeddings,
while also providing a negative result for leverage score sampling.

Statistical analysis of random design
models has a long history, with important connections to both
asymptotic and non-asymptotic random matrix theory
\citep[e.g., see][]{bai2010spectral,vershynin2018high}. A
variety of random designs have 
been considered, of which the most relevant to our work are Gaussian
and sub-gaussian designs. These models have proven extremely useful in
understanding the performance of a variety of linear regression
estimators \citep{DW15_TR,HMRT19_TR,bayati2011lasso,miolane2021distribution},
sample covariance estimators 
\citep{koltchinskii2017concentration,ledoit2011eigenvectors} and others. Here, a sub-gaussian random design typically refers to a
matrix $\Z\Sigmab^{1/2}$, where $\Z$ consists of independent rows
which have either bounded sub-gaussian norm or
i.i.d.~sub-gaussian entries, or more generally, satisfy a
concentration property for a family of Lipschitz functions
\citep[e.g.,][]{louart2018concentration}. Our gaussianization framework opens the
possibility of extending many of these results to sketching.

Our results are related to the study of the rates of convergence in
the multivariate central limit theorem (CLT), with the key difference that we relax
our notion of gaussianity, which allows us to use a stronger notion
of distance (i.e., the total variation distance, together with the
Hanson-Wright constant). Nevertheless, it is helpful to compare
our approach with the convergence rates for the multivariate CLT in
terms of the Wasserstein distance
\citep{chen2005stein,bonis2020stein,fang2022p}. Using the setup from our key
technical result,
Theorem \ref{t:main}, for a random vector $\x$ with covariance matrix $\Sigmab=\E[\x\x^\top]$ that satisfies
$\|\Sigmab^{-1/2}\x\|\leq M$ almost surely, the currently best known
Wasserstein CLT of order 2 yields $W_2(\xbt,\Nc(\zero,\Sigmab)) =
O(M\sqrt{d\log(k)/k})$ \citep{eldan2020clt}, whereas our bound on
the Hanson-Wright constant is $O(1+M\sqrt{\log(M)/k})$. The guarantees are not directly comparable,
but one could argue that both bounds are useful primarily when $k$ is
large enough to absorb the dependence on $M$ and $d$, i.e.,
$k=O(M^2d\log(M))$ for the Wasserstein bound  and $k=O(M^2\log(M))$
for our result.
Also, while the Wasserstein distance can be
used to bound the difference between expectations of Lipschitz
functions, it is non-trivial to effectively bound the Lipschitz
constant for most of the quantities considered in this work, such as
the least squares approximation error, and using 
this approach would require a more complex case-by-case analysis.


\vspace{-1mm}
\section{Conclusions and open questions}

We provided the first general characterization of
Algorithmic Gaussianization, which refers to
the phenomenon that many algorithmic techniques used to construct
small random
representations (sketches) of large datasets produce data samples that
are more Gaussian-like than the original data. In our main
result, we showed that a sparse sketching matrix can be used to
produce a sample that is nearly 
indistinguishable from a sub-gaussian random design. We used this to
show a reduction between a fast sketching technique called LESS
embeddings and a sub-gaussian random matrix whose rows satisfy the
classical Hanson-Wright inequality. Our lower bound
showed that the result is nearly tight up to logarithmic factors, and
that the sub-gaussian reduction is not possible for leverage score sampling.
We demonstrated how our techniques can 
be used to provide improved guarantees for sketched estimators
in least squares, Lasso regression, and low-rank approximation. 
More broadly, our results point to new open questions
related to the complexity of generating repeated samples from a product
between a matrix and a Gaussian-like random vector:
\paragraph{Open question.} Given an $N\times d$ matrix $\A$ and a function
family $\Fc$ in $\R^d\rightarrow \R$ (as in Figure~\ref{f:hierarchy}), what is the
complexity of producing $n$ independent samples of a $d$-dimensional
random vector that is, up to total variation
distance~$\delta$, distributed according to $ \A^\top\z$,  
where
$\z$ is isotropic and has the concentration property over $\Fc$ with
constant $K=O(1)$, as given by Definition~\ref{d:concentration}.

\paragraph{Acknowledgments.}
Thanks to Edgar Dobriban and Michael Mahoney for valuable discussions
regarding this paper.

\bibliographystyle{plain}
\bibliography{../pap}

\begin{thebibliography}{10}

\bibitem{adamczak2015note}
Radoslaw Adamczak.
\newblock A note on the hanson-wright inequality for random vectors with
  dependencies.
\newblock {\em Electronic Communications in Probability}, 20:1--13, 2015.

\bibitem{ailon2009fast}
Nir Ailon and Bernard Chazelle.
\newblock The fast johnson--lindenstrauss transform and approximate nearest
  neighbors.
\newblock {\em SIAM Journal on computing}, 39(1):302--322, 2009.

\bibitem{ridge-leverage-scores}
Ahmed~El Alaoui and Michael~W. Mahoney.
\newblock Fast randomized kernel ridge regression with statistical guarantees.
\newblock In {\em Proceedings of the 28th International Conference on Neural
  Information Processing Systems}, pages 775--783, Montreal, Canada, December
  2015.

\bibitem{bai2010spectral}
Zhidong Bai and Jack~W Silverstein.
\newblock {\em Spectral analysis of large dimensional random matrices},
  volume~20.
\newblock Springer, 2010.

\bibitem{bamberger2021hanson}
Stefan Bamberger, Felix Krahmer, and Rachel Ward.
\newblock The hanson-wright inequality for random tensors.
\newblock {\em arXiv preprint arXiv:2106.13345}, 2021.

\bibitem{bayati2011lasso}
Mohsen Bayati and Andrea Montanari.
\newblock The lasso risk for gaussian matrices.
\newblock {\em IEEE Transactions on Information Theory}, 58(4):1997--2017,
  2011.

\bibitem{bonis2020stein}
Thomas Bonis.
\newblock Stein’s method for normal approximation in wasserstein distances
  with application to the multivariate central limit theorem.
\newblock {\em Probability Theory and Related Fields}, 178(3):827--860, 2020.

\bibitem{libsvm}
Chih-Chung Chang and Chih-Jen Lin.
\newblock {LIBSVM}: A library for support vector machines.
\newblock {\em ACM Transactions on Intelligent Systems and Technology},
  2:27:1--27:27, 2011.

\bibitem{chen2005stein}
Louis~HY Chen and Qi-Man Shao.
\newblock Stein’s method for normal approximation.
\newblock {\em An introduction to Stein’s method}, 4:1--59, 2005.

\bibitem{chen2017condition}
Xue Chen and Eric Price.
\newblock Active regression via linear-sample sparsification.
\newblock In {\em Proceedings of the 32nd Conference on Learning Theory}, 2019.

\bibitem{cw-sparse}
Kenneth~L. Clarkson and David~P. Woodruff.
\newblock Low-rank approximation and regression in input sparsity time.
\newblock {\em J. ACM}, 63(6):54:1--54:45, January 2017.

\bibitem{dpp-intermediate}
Micha{\l} Derezi{\'n}ski.
\newblock Fast determinantal point processes via distortion-free intermediate
  sampling.
\newblock In {\em Proceedings of the Thirty-Second Conference on Learning
  Theory}, volume~99, pages 1029--1049, 2019.

\bibitem{minimax-experimental-design}
Micha{\l} Derezi{\'n}ski, Kenneth~L. Clarkson, Michael~W. Mahoney, and
  Manfred~K. Warmuth.
\newblock Minimax experimental design: Bridging the gap between statistical and
  worst-case approaches to least squares regression.
\newblock In {\em Proceedings of the Thirty-Second Conference on Learning
  Theory}, volume~99, pages 1050--1069, 2019.

\bibitem{nystrom-multiple-descent}
Micha{\l} Derezi{\'n}ski, Rajiv Khanna, and Michael~W Mahoney.
\newblock Improved guarantees and a multiple-descent curve for the column
  subset selection problem and the nystr\"om method.
\newblock In {\em Advances in Neural Information Processing Systems},
  volume~33, pages 4953--4964, 2020.

\bibitem{newton-less}
Micha{\l} Derezi{\'n}ski, Jonathan Lacotte, Mert Pilanci, and Michael~W
  Mahoney.
\newblock Newton-{LESS}: Sparsification without trade-offs for the sketched
  newton update.
\newblock In {\em Advances in Neural Information Processing Systems},
  volume~34, pages 2835--2847, 2021.

\bibitem{precise-expressions}
Micha{\l} Derezi{\'n}ski, Feynman Liang, Zhenyu Liao, and Michael~W Mahoney.
\newblock Precise expressions for random projections: Low-rank approximation
  and randomized {N}ewton.
\newblock In {\em Advances in Neural Information Processing Systems},
  volume~33, pages 18272--18283, 2020.

\bibitem{derezinski2020bayesian}
Micha{\l} Derezi\'nski, Feynman Liang, and Michael Mahoney.
\newblock Bayesian experimental design using regularized determinantal point
  processes.
\newblock In {\em International Conference on Artificial Intelligence and
  Statistics}, pages 3197--3207, 2020.

\bibitem{less-embeddings}
Micha{\l} Derezi{\'n}ski, Zhenyu Liao, Edgar Dobriban, and Michael~W Mahoney.
\newblock Sparse sketches with small inversion bias.
\newblock In {\em Proceedings of the 34th Conference on Learning Theory}, 2021.

\bibitem{DM21_NoticesAMS}
Micha{\l} Derezi{\'n}ski and Michael~W Mahoney.
\newblock Determinantal point processes in randomized numerical linear algebra.
\newblock {\em Notices of the American Mathematical Society}, 68(1):34--45,
  2021.

\bibitem{dobriban2019asymptotics}
Edgar Dobriban and Sifan Liu.
\newblock Asymptotics for sketching in least squares regression.
\newblock In {\em Advances in Neural Information Processing Systems}, pages
  3675--3685, 2019.

\bibitem{DW15_TR}
Edgar Dobriban and Stefan Wager.
\newblock High-dimensional asymptotics of prediction: Ridge regression and
  classification.
\newblock {\em The Annals of Statistics}, 46(1):247--279, 2018.

\bibitem{fast-leverage-scores}
Petros Drineas, Malik Magdon-Ismail, Michael~W. Mahoney, and David~P. Woodruff.
\newblock Fast approximation of matrix coherence and statistical leverage.
\newblock {\em J. Mach. Learn. Res.}, 13(1):3475--3506, December 2012.

\bibitem{DM16_CACM}
Petros Drineas and Michael~W. Mahoney.
\newblock {RandNLA}: Randomized numerical linear algebra.
\newblock {\em Communications of the ACM}, 59:80--90, 2016.

\bibitem{drineas2006sampling}
Petros Drineas, Michael~W Mahoney, and S.~Muthukrishnan.
\newblock Sampling algorithms for $\ell_2$ regression and applications.
\newblock In {\em Proceedings of the Symposium on Discrete algorithms}, pages
  1127--1136, 2006.

\bibitem{eldan2020clt}
Ronen Eldan, Dan Mikulincer, and Alex Zhai.
\newblock The clt in high dimensions: quantitative bounds via martingale
  embedding.
\newblock {\em The Annals of Probability}, 48(5):2494--2524, 2020.

\bibitem{fang2022p}
Xiao Fang and Yuta Koike.
\newblock From $ p $-wasserstein bounds to moderate deviations.
\newblock {\em arXiv preprint arXiv:2205.13307}, 2022.

\bibitem{tropp2011structure}
Nathan Halko, Per-Gunnar Martinsson, and Joel~A Tropp.
\newblock Finding structure with randomness: Probabilistic algorithms for
  constructing approximate matrix decompositions.
\newblock {\em SIAM review}, 53(2):217--288, 2011.

\bibitem{HMRT19_TR}
T.~Hastie, A.~Montanari, S.~Rosset, and R.~J. Tibshirani.
\newblock Surprises in high-dimensional ridgeless least squares interpolation.
\newblock Technical Report Preprint: arXiv:1903.08560, 2019.

\bibitem{hitczenko1990best}
Pawel Hitczenko.
\newblock Best constants in martingale version of rosenthal's inequality.
\newblock {\em The Annals of Probability}, 18(4):1656--1668, 1990.

\bibitem{hsu2012tail}
Daniel Hsu, Sham Kakade, and Tong Zhang.
\newblock A tail inequality for quadratic forms of subgaussian random vectors.
\newblock {\em Electronic Communications in Probability}, 17, 2012.

\bibitem{koltchinskii2017concentration}
Vladimir Koltchinskii and Karim Lounici.
\newblock Concentration inequalities and moment bounds for sample covariance
  operators.
\newblock {\em Bernoulli}, 23(1):110--133, 2017.

\bibitem{lacotte2020optimal}
Jonathan Lacotte, Sifan Liu, Edgar Dobriban, and Mert Pilanci.
\newblock Optimal iterative sketching methods with the subsampled randomized
  hadamard transform.
\newblock {\em Advances in Neural Information Processing Systems},
  33:9725--9735, 2020.

\bibitem{ledoit2011eigenvectors}
Olivier Ledoit and Sandrine P{\'e}ch{\'e}.
\newblock Eigenvectors of some large sample covariance matrix ensembles.
\newblock {\em Probability Theory and Related Fields}, 151(1-2):233--264, 2011.

\bibitem{lopes2019bootstrapping}
Miles~E Lopes, N~Benjamin Erichson, and Michael~W Mahoney.
\newblock Bootstrapping the operator norm in high dimensions: Error estimation
  for covariance matrices and sketching.
\newblock {\em arXiv preprint arXiv:1909.06120}, 2019.

\bibitem{louart2018concentration}
Cosme Louart and Romain Couillet.
\newblock Concentration of measure and large random matrices with an
  application to sample covariance matrices.
\newblock {\em arXiv preprint arXiv:1805.08295}, 2018.

\bibitem{ping-ma2014}
Ping Ma, Michael Mahoney, and Bin Yu.
\newblock A statistical perspective on algorithmic leveraging.
\newblock In {\em Proceedings of the 31st International Conference on Machine
  Learning}, pages 91--99, 2014.

\bibitem{mendelson2007reconstruction}
Shahar Mendelson, Alain Pajor, and Nicole Tomczak-Jaegermann.
\newblock Reconstruction and subgaussian operators in asymptotic geometric
  analysis.
\newblock {\em Geometric and Functional Analysis}, 17(4):1248--1282, 2007.

\bibitem{miolane2021distribution}
L{\'e}o Miolane and Andrea Montanari.
\newblock The distribution of the lasso: Uniform control over sparse balls and
  adaptive parameter tuning.
\newblock {\em The Annals of Statistics}, 49(4):2313--2335, 2021.

\bibitem{randomized-newton}
Mojm\'ir Mutn\'y, Micha{\l} Derezi\'nski, and Andreas Krause.
\newblock Convergence analysis of block coordinate algorithms with
  determinantal sampling.
\newblock In {\em International Conference on Artificial Intelligence and
  Statistics}, pages 3110--3120, 2020.

\bibitem{nn-sparse}
Jelani Nelson and Huy~L. Nguy\^{e}n.
\newblock {OSNAP}: Faster numerical linear algebra algorithms via sparser
  subspace embeddings.
\newblock In {\em Proceedings of the 2013 IEEE 54th Annual Symposium on
  Foundations of Computer Science}, FOCS '13, pages 117--126. IEEE Computer
  Society, 2013.

\bibitem{pilanci2015randomized}
Mert Pilanci and Martin~J Wainwright.
\newblock Randomized sketches of convex programs with sharp guarantees.
\newblock {\em IEEE Transactions on Information Theory}, 61(9):5096--5115,
  2015.

\bibitem{rudelson2013hanson}
Mark Rudelson and Roman Vershynin.
\newblock {Hanson-Wright} inequality and sub-gaussian concentration.
\newblock {\em Electronic Communications in Probability}, 18, 2013.

\bibitem{sarlos-sketching}
Tamas Sarlos.
\newblock Improved approximation algorithms for large matrices via random
  projections.
\newblock In {\em Proceedings of the 47th Symposium on Foundations of Computer
  Science}, pages 143--152, 2006.

\bibitem{tropp2011improved}
Joel~A Tropp.
\newblock Improved analysis of the subsampled randomized hadamard transform.
\newblock {\em Advances in Adaptive Data Analysis}, 3(01n02):115--126, 2011.

\bibitem{matrix-tail-bounds}
Joel~A. Tropp.
\newblock User-friendly tail bounds for sums of random matrices.
\newblock {\em Foundations of Computational Mathematics}, 12(4):389--434,
  August 2012.

\bibitem{vershynin2018high}
Roman Vershynin.
\newblock {\em High-dimensional probability: An introduction with applications
  in data science}, volume~47.
\newblock Cambridge university press, 2018.

\bibitem{vershynin2020concentration}
Roman Vershynin.
\newblock Concentration inequalities for random tensors.
\newblock {\em Bernoulli}, 26(4):3139--3162, 2020.

\bibitem{sketched-ridge-regression}
Shusen Wang, Alex Gittens, and Michael~W. Mahoney.
\newblock Sketched ridge regression: Optimization perspective, statistical
  perspective, and model averaging.
\newblock In {\em Proceedings of the 34th International Conference on Machine
  Learning}, pages 3608--3616, 06--11 Aug 2017.

\bibitem{woodruff2014sketching}
David~P. Woodruff.
\newblock Sketching as a tool for numerical linear algebra.
\newblock {\em Foundations and Trends{\textregistered} in Theoretical Computer
  Science}, 10(1--2):1--157, 2014.

\end{thebibliography}

\appendix

\section{More applications:  Randomized SVD and low-distortion
  embeddings}
\label{s:examples}
In this section, we provide further examples of how Theorem \ref{c:less-embeddings} can be applied in
combination with existing results for sub-gaussian random
designs. In each of these examples, we can extend existing
results from sub-gaussian embeddings (which are dense, and
therefore not very efficient) to analogous results for LESS embeddings
(which are sparse, and thus can be implemented much more efficiently).

\subsection{Randomized SVD}
An important application of randomized sketching is low-rank approximation, where
our goal is to estimate a small number of top principal directions
of an $N\times d$ data matrix $\A$. One of the most popular techniques
in this area is the  Randomized SVD algorithm
\citep{tropp2011structure}, where we construct the approximation by
projecting the dataset onto a subspace spanned by the rows of the
sketch $\S\A$, with $\S$ denoting the $n\times N$ sketching matrix,
and $n\ll d$. The error of such an approximation is often measured by the
sum of squared lengths of the residuals from the projection:
$\|\A - \A\cdot\mathrm{Proj}_{\S\A}\|_F^2$, where
$\mathrm{Proj}_{\S\A}$ denotes the projection onto the row-span of
$\S\A$ and $\|\cdot\|_F$ is the Frobenius norm. Prior work has shown strong approximation guarantees for
Gaussian embeddings in this context, such as
the following, given by \cite{tropp2011structure}, comparing the expected error to the best rank $k$
approximation, for some $k<n-2$:
\begin{align*}
  \E\Big[\big\|\A -
    \A\cdot\mathrm{Proj}_{\S\A}\big\|_F^2\Big]\leq
  \Big(1+\frac{k}{n-k-1}\Big)\cdot\min_{\B:\,\rank(\B)=k}\|\A-\B\|_F^2.
\end{align*}
Those guarantees were later extended (and, in some regimes, improved) for sub-gaussian embeddings \citep{precise-expressions} and
Determinantal Point Processes \citep{nystrom-multiple-descent}.
The results for Gaussian and sub-gaussian embeddings rely heavily on the
Hanson-Wright property of the rows of the sketching
matrices. Thus, our techniques can be used to extend them to LESS
embeddings by using
Theorem~\ref{c:less-embeddings}. Consequently, we can obtain bounds for the
expected approximation 
error of Randomized SVD with LESS embeddings that, in a certain regime
of small sketch sizes, nearly match those achieved by Gaussian embeddings.
\begin{corollary}\label{c:rand-svd}
  Consider an $N\times d$ matrix $\A$ with stable rank
  $r=\|\A\|_F^2/\|\A\|^2$.
  Given a LESS embedding $\S$ of size $n\leq r/2$, and a
  Gaussian embedding $\St$ of the same size, we have: 
  \vspace{-1mm}
  \begin{align*}
    \E\Big[\big\|\A -
    \A\cdot\mathrm{Proj}_{\S\A}\big\|_F^2\Big]\approx_{1+\epsilon}
       \E\Big[\big\|\A -
    \A\cdot\mathrm{Proj}_{\St\A}\big\|_F^2\Big]\quad\text{with}\quad \epsilon=\tilde O(1/\sqrt r).
  \end{align*}
\end{corollary}
\begin{proof}
We rely on Theorem~2 of \cite{precise-expressions},
  characterizing the expected approximation error for a sub-gaussian
  embedding of size $n$ via an implicit analytic~formula:
  \begin{align}
    \mathbb{E}\Big[\big\|\A -
    \A\cdot\mathrm{Proj}_{\St\A}\big\|_F^2\Big] \approx_{1+\epsilon}
    n\lambda_n, \quad\text{where}\quad \lambda_n = f_{\A}^{-1}(n)\quad\text{for}\quad f_{\A}(\lambda)=\tr\,\A^\top\A(\A^\top\A+\lambda\I)^{-1}.\label{eq:implicit}
  \end{align}
  Here, $f_{\A}^{-1}(n)$ denotes the function inverse of
  $f_{\A}$ at $n$. The function $f_{\A}(\lambda)$ is known in the literature as
  $\lambda$-statistical dimension, a notion of degrees of freedom that
  arises in the analysis of ridge regression%
  \ifisarxiv, e.g., see \cite{derezinski2020bayesian}. \else. \fi
  Naturally, their
  result applies to the Gaussian embedding $\St\A$, but also to any
  random design $\Z\Sigmab^{1/2}$ (where $\Sigmab=\A^\top\A$) such
  that the rows satisfy the Hanson-Wright inequality with a constant
  $K$ (with $\epsilon$ having a polynomial dependence on $K$). So, we can
  apply our reduction to LESS embeddings by coupling $\sqrt n\S\A$
  with $\Z\Sigmabt^{1/2}$ for some
  $\Sigmabt\approx_{1+\delta}\Sigmab$, as in
  Theorem~\ref{c:less-embeddings}. Letting 
  $\Er(\S\A)$ denote $\|\A-\A\cdot\mathrm{Proj}_{\S\A}\|_F^2$ and
  $\Ec$ be the $(1-\delta)$-probability event where $\sqrt 
  n\S\A=\Z\Sigmabt^{1/2}$, we have:
  \begin{align*}
    \big|\E\big[\Er(\S\A)\big] - \E\big[\Er(\Z\Sigmabt^{1/2})\big]\big|
    &=
      \delta\cdot \big|\E\big[\Er(\S\A) -
      \Er(\Z\Sigmabt^{1/2})\mid\neg\Ec\big]\big|\leq \delta\cdot \|\A\|_F^2.
  \end{align*}
Since $\|\A\|_F^2\leq r\cdot \|\A\|^2\leq \frac{r}{r-n}\cdot
\min_{\Z}\Er(\Z\Sigmabt^{1/2})$, setting $\delta \leq
\epsilon/2$, we get
$\E\big[\Er(\S\A)\big]\approx_{1+\epsilon}\E\big[\Er(\Z\Sigmabt^{1/2})\big]$. Next,
since $\Sigmabt\approx_{1+\delta}\Sigmab$, we can use the implicit
error formula for sketching $\Sigmabt^{1/2}$ with $\Z$, and relate that back
to the original problem as follows:
\begin{align*}
  \E\big[\Er(\Z\Sigmabt^{1/2})\big]
  \approx_{1+\delta}\E\big[\|\Sigmabt^{1/2}(\I-\mathrm{Proj}_{\Z\Sigmabt})\|_{F}^2\big]
  \approx_{1+\epsilon} n\tilde\lambda_n,
\end{align*}
where $\tilde\lambda_n$ is defined as in \eqref{eq:implicit}, but with
$\A^\top\A$ replaced by $\Sigmabt$. To close the loop, we must relate
the implicit analytic expression based on $\Sigmabt$ back to
$\lambda_n$. This can be done easily by bounding the derivative of
$f_{\A}^{-1}$, showing that $\tilde\lambda_n\approx_{1+O(\delta)}\lambda_n$.
Now, it remains to observe that since both
the LESS embedding error $\E\big[\Er(\S\A)\big]$ and the
Gaussian embedding error $\E\big[\Er(\St\A)\big]$ are
approximated by the same quantity $n\lambda_n$ up to a $1+O(\epsilon)$
factor, they are also approximated by each other, concluding the proof.
\end{proof}

\subsection{Low-distortion embeddings}
\label{a:low-distortion}
Another important property of sub-gaussian random matrices is that
they can be used to construct low-dimensional embeddings that preserve the
geometry of a high-dimensional space, such as the subspace embedding
property and the Johnson-Lindenstrauss property, which are central to many
applications of sketching \citep[e.g., see][]{woodruff2014sketching}.  We briefly mention some
classical examples and discuss how they can be extended to
gaussianized sketches, including LESS embeddings. 

We say that an $n\times N$ matrix $\S$ is an $\epsilon$-low-distortion embedding
for some set of points $\mathcal X\subseteq \R^N$ if there is a fixed 
scalar $\alpha>0$ such that $\|\alpha\S\v\|\approx_{1+\epsilon}\|\v\|$
for all $\v\in\mathcal X$,
i.e., the embedding approximately preserves the Euclidean norm of
$\v$. When $\S$ is a sketching matrix applied to an $N\times d$ data
matrix $\A$, we are typically interested in vectors $\v$ from the
column span of $\A$ (i.e., such that $\v=\A\x$ for some $\x\in\R^d$),
denoted by $\mathrm{span}(\A)$.
Relying on Theorem~\ref{c:less-embeddings} and standard
properties of sub-gaussian random matrices \cite{vershynin2018high}, we can establish the
following low-distortion embedding properties for LESS embeddings.
\begin{corollary}\label{c:low-distortion}
Let $\A$ be $N\times d$, and $\S$ be a LESS embedding for
$\A$ of size $n$ with $d\log(nd/(\epsilon\delta))$ non-zeros per
row. Then, there is an absolute constant $C$ such that the following claims are true:
\begin{enumerate}
  \item \emph{(Johnson-Lindenstrauss property)}\quad \ For any finite set
    $\mathcal X\subseteq \mathrm{span}(\A)$, if $n\geq
   C\log(|\mathcal X|/\delta)/\epsilon^2$, then with probability $1-\delta$,
   matrix $\S$ is an $\epsilon$-low-distortion embedding for $\mathcal X$.
   \item \emph{(Subspace embedding property)}\quad \ If $n\geq
     C(d+\log(1/\delta))/\epsilon^2$, then with probability
     $1-\delta$, matrix $\S$ is an $\epsilon$-low-distortion embedding for
     $\mathcal X=\mathrm{span}(\A)$.
\end{enumerate}
\end{corollary}
\begin{proof}
Whenever, $\mathcal X\subseteq\mathrm{span}(\A)$, the low-distortion property can be
formulated as a property of the sketch $\S\A$, treated as a linear
transformation: for any $\v\in\mathrm{span}(\A)$, we have $\v=\A\x$
for some $\x\in\R^d$, and we need to ensure that $\|\alpha\S\A\x\|\approx_{1+\epsilon}\|\A\x\|$.
  Here, we can once again rely on Theorem~\ref{c:less-embeddings},
  coupling $\sqrt n\S\A$ with a sub-gaussian design $\Z\Sigmabt^{1/2}$
  for some $\Sigmabt\approx_{1+\delta}\A^\top\A$. If $\Z$ is an
$\epsilon$-low-distortion embedding for $\widetilde{\mathcal
  X}=\{\Sigmabt^{1/2}\x:\ \A\x\in\mathcal X\}$ with constant
$\alpha$, then with probability $1-\delta$ (specifically, when 
$\sqrt n\S\A=\Z\Sigmabt^{1/2}$), for any $\A\x\in\mathcal X$, we have:
\begin{align*}
\|\alpha\sqrt n\S\A\x\| = \|\alpha\Z\Sigmabt^{1/2}\x\|
  \approx_{1+\epsilon}\|\Sigmabt^{1/2}\x\|\approx_{1+\delta}\|\A\x\|, 
\end{align*}
where in the last step we used that
$\Sigmabt\approx_{1+\delta}\A^\top\A$. Thus, as long as $\delta\leq
\epsilon$, we reduced the problem to showing the corresponding
property for $\Z$. Both the Johnson-Lindenstrauss and subspace
embedding properties can be shown for sub-gaussian embeddings (with
sizes as given in the statement) using standard arguments (e.g., see 
the proof of  Theorem 4.6.1~in \cite{vershynin2018high}), thus
concluding the proof.
\end{proof}
We note that low-distortion embedding properties have been shown for
other sketching operators, including SRHTs \citep{tropp2011improved} and other sparse sketching
matrices like CountSketch \citep{cw-sparse} and SJLTs
\citep{nn-sparse}. These results often circumvent 
sub-gaussian analysis by relying on matrix concentration inequalities
such as the matrix Bernstein inequality (Lemma 
\ref{l:bounded-Bernstein}) or other more involved arguments, however
this inevitably incurs additional 
overhead factors in the required sketch size $n$. For example, an SRHT
requires sketch size of at least $n\geq O(d\log(d)/\epsilon^2)$ to
establish the subspace embedding property, instead of $O(d/\epsilon^2)$
shown above (the additional log factor is a well-known limitation of
the matrix concentration-style analysis). The
$O(d\log(d)/\epsilon^2)$ guarantee was also previously shown for LESS
embeddings with $d$ non-zeros per row, also using matrix concentration
inequalities instead of sub-gaussian concentration
\cite[Lemma~12]{less-embeddings}. Thus, the above corollary suggests  
that, when requiring a subspace embedding, we can trade an additional
log-factor in the sketch size for a log-factor in the density of a
LESS embedding matrix. This could be a useful trade-off when we wish
to minimize the storage space of the sketched data.

  \section{Experiments}
  \label{s:experiments}
In this section, we aim to evaluate the degree of Algorithmic
Gaussianization for various sketching methods, using the sketched
least squares task from Section \ref{s:main-app} as an example. In
this task, we are given an $N\times d$ matrix $\A$ and an
$N$-dimensional vector $\b$ which define the regression loss
$L(\w)=\|\A\w-\b\|^2$. We then use an $n\times N$ sketching 
matrix $\S$ to generate a sketched least squares estimate $\wbh =
\argmin_{\w}\|\S(\A\w-\b)\|^2$ of the exact solution $\w^*$. Our goal
is to compare the normalized expected approximation error $\E[L(\wbh)-L(\w^*)]/L(\w^*)$
of the sketched estimate to the problem-independent expression
$\frac{d}{n-d-1}$ achieved by a Gaussian embedding.

\begin{figure} 
\centering     
\subfigure{\label{fig:1}\includegraphics[width=.43\textwidth]{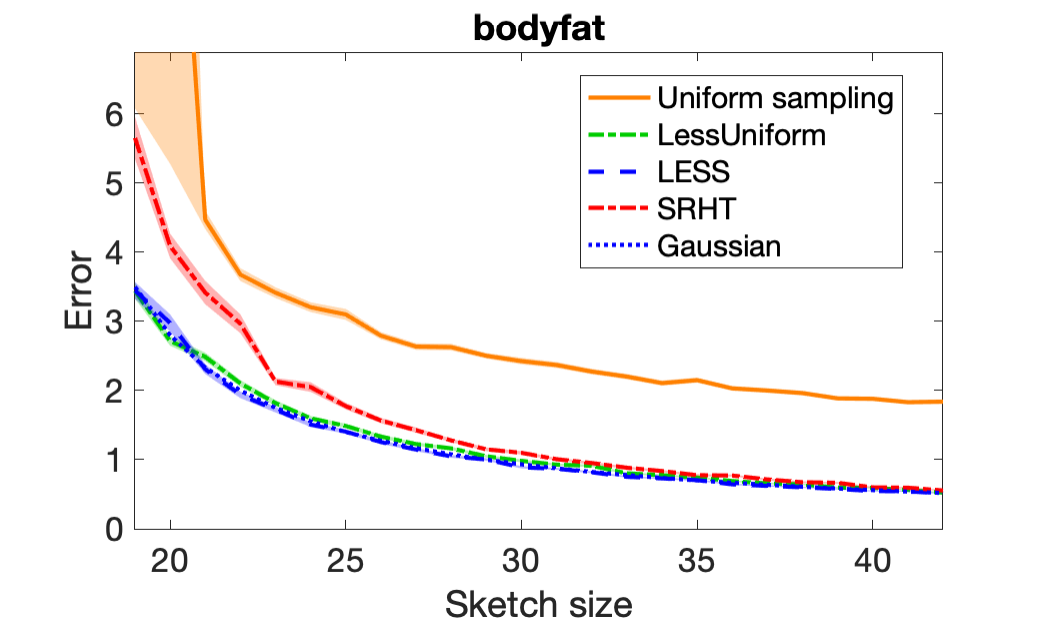}}
\subfigure{\label{fig:2}\includegraphics[width=.43\textwidth]{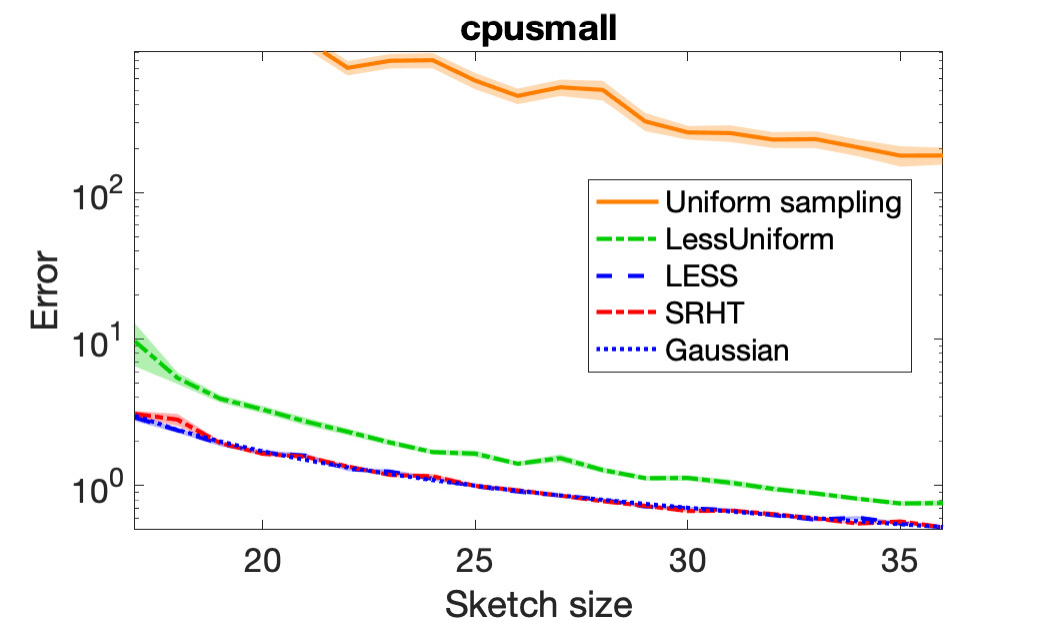}}
\subfigure{\label{fig:3}\includegraphics[width=.43\textwidth]{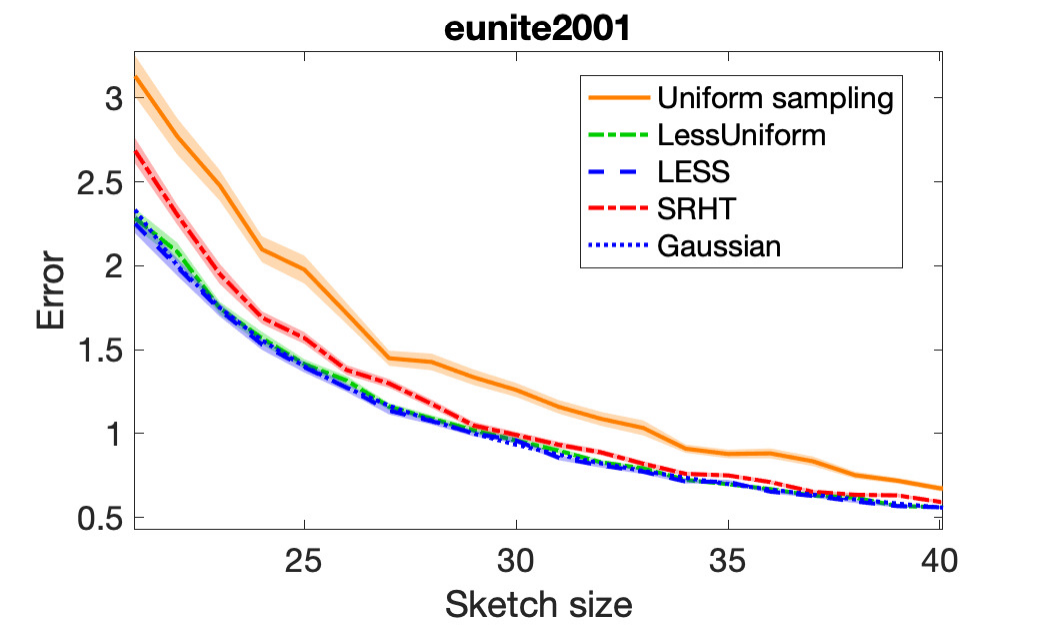}}
\subfigure{\label{fig:4}\includegraphics[width=.43\textwidth]{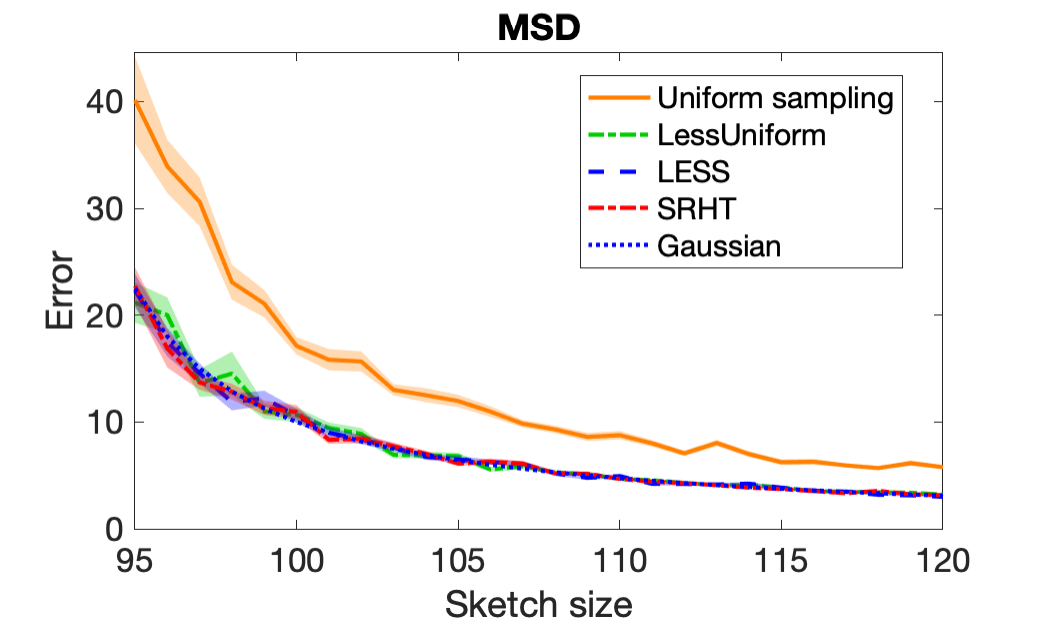}}
\vspace{-3mm}
\caption{Comparison of expected approximation error for sketched
  least squares on Libsvm datasets \citep{libsvm}, with shading
  indicating standard error of the mean, and ``Gaussian''
  showing the theoretical $\frac{d}{n-d-1}$ expression from
  Theorem \ref{t:main-app}.}\label{fig:plots}
\vspace{-5mm}
\end{figure}

In Figure \ref{fig:plots}, we plot the empirically estimated expected approximation error
with varying sketch size for several sketching techniques on four
benchmark datasets. We compare them against the expected error for
Gaussian embeddings, which is given by \eqref{eq:ols-formula}.  Two of
the sketching techniques, Subsampled Randomized Hadamard Transform
\citep[SRHT,][]{ailon2009fast} and Leverage Score Sparsified embeddings
\citep[LESS,][]{less-embeddings}, have strong
theoretical guarantees for the expected error, and have comparable
nearly-linear time complexity.\footnote{Recall that the complexity of sketched least squares with
LESS is $O(\nnz(\A)\log N + nd^2)$, compared to $O(Nd\log N + nd^2)$
for SRHT. So, the complexity of LESS is better than SRHT when $\A$ is a sparse matrix, i.e., when
$\nnz(\A)\ll Nd$, and otherwise, the complexities are comparable.} However, there are some key differences: LESS has a precise non-asymptotic 
guarantee given in this work (Theorem~\ref{t:main-app}), whereas SRHT only has a non-asymptotic
upper bound \citep{sarlos-sketching} and a precise
asymptotic guarantee under some additional assumptions
\citep{dobriban2019asymptotics}. We compare these two sketches
with two much cheaper baselines: (1) uniform sampling, and (2) a simplified
variant of LESS, called LessUniform \citep{newton-less}, which
eliminates the leverage score approximation preprocessing step (but
still uses $d$ non-zeros per row).
The overall cost for both LessUniform and for uniform
sampling is the same as the cost of solving the sketched sub-problem,
i.e., $O(nd^2)$, where $n$ is the sketch size. Our
Theorems~\ref{c:less-embeddings} and \ref{t:main-app} can also be applied to
LessUniform, but  
the Hanson-Wright constant becomes dependent on the maximum leverage
score of matrix $\A$ (i.e., the coherence of $\A$; see Remark \ref{r:less}).

From Figure \ref{fig:plots}, we first confirm that LESS enjoys a Gaussian-like problem-independent expected approximation error that
matches our theory. Remarkably, we can verify this for all datasets
and all sketch sizes, downto the precision of our
empirical mean estimates. Thus, these results suggest that the Gaussian error
estimate for LESS embeddings is even more accurate and broadly applicable than promised by
our theory. Next, we observe that for small sketch sizes, SRHT does
not always exhibit Gaussian-like error (see left two plots), but
as the sketch size increases, the gap decreases. This appears roughly in line
with the existing asymptotic theory for SRHTs. Finally, LessUniform
performs much better than uniform sampling, despite having
the same time complexity. The error for LessUniform appears Gaussian-like in three
out of four cases. In the case of dataset \textit{cpusmall},
the problem exhibits very high coherence, and as a result both uniform sampling
and LessUniform perform worse. This aligns with
Theorem~\ref{c:less-embeddings} applied to LessUniform, since the Hanson-Wright
constant scales with the largest leverage score of $\A$.

\section{Notation and preliminaries for the proofs}
\label{s:preliminaries}

\paragraph{Notation.}
We say that $a\approx_{\alpha}b$ for $\alpha\geq 1$, if
$b/\alpha\leq a\leq \alpha b$. We analogously define $\A\approx_\alpha\B$ for
positive semidefinite matrices using the Loewner ordering. We use
$\|\A\|_F=\sqrt{\tr(\A^\top\A)}$ to denote the Frobenius norm and
$\|\A\|$ to denote the spectral norm.  Also, we
let $a\lesssim b$ denote that there exists an absolute constant $C>0$ such
that $a\leq Cb$, and $a=\poly(b,c)$ means that $a$ is bounded by a
polynomial function of $b$ and $c$. Moreover, we use $\tilde O(\cdot)$ to refer to
big-O notation where polylogarithmic terms are ignored.  For random variables/vectors $X$ and $Y$ defined over the same domain with
  measures $\mu$ and $\nu$, respectively, we
  define the total variation distance between them as
    $d_{\tv}(X,Y) = \sup_{E\in \mathcal B} |\mu(E) - \nu(E)|$,
  where $\mathcal B$ denotes all measurable events. Note that the
  total variation distance can also be defined as the infimum over
  $\delta$ such that there exists a coupling between $X$ and $Y$ for
  which $\Pr(X\neq Y)=\delta$. Finally, we define the sub-gaussian
  Orlicz norm as: $\|X\|_{\psi_2} = \inf\{t>0:\,\E\,\exp(X^2/t^2)\leq 2\}$.

  \paragraph{Matrix concentration inequalities.}
  In the proof of our main results, we use the following versions of the matrix Bernstein concentration
  inequality for the sums of independent symmetric random matrices.
  \begin{lemma}[Sub-exponential matrix Bernstein]{\cite[Theorem~6.2]{matrix-tail-bounds}}\label{l:sub-exp-Bernstein}
For $i=1,2,...$, consider a finite sequence $\M_i$ of $d\times d$
independent and symmetric random matrices such that   
\[
	\E[\M_i] = \mathbf{0}, \quad \E[\M_i^p] \preceq \frac{p!}2
        \cdot R^{p-2} \A_i^2\quad\textmd{for}\quad p=2,3,...
\]
Then, defining the variance parameter $\sigma^2 = \| \sum_i
\A_i^2 \|$, for any $t>0$ we have:
\begin{align*}
	\Pr \bigg\{ \lambda_{\max}\Big( \sum\nolimits_i \M_i \Big) \geq t
  \bigg\}& \leq d \cdot \exp \left( \frac{ -t^2/2 }{ \sigma^2 +
            R t } \right).
\end{align*}
\end{lemma}
  \begin{lemma}[Bounded matrix Bernstein]{\cite[Theorem~6.1]{matrix-tail-bounds}}\label{l:bounded-Bernstein}
For $i=1,2,...$, consider a finite sequence $\M_i$ of $d\times d$
independent and symmetric random matrices such that   
\[
	\E[\M_i] = \mathbf{0}, \quad \lambda_{\max}(\M_i)\leq
        R\qquad\text{almost surely.}
\]
Then, defining the variance parameter $\sigma^2 = \| \sum_i
\E[\M_i^2]\|$, for any $t>0$ we have:
\begin{align*}
	\Pr \bigg\{ \lambda_{\max}\Big( \sum\nolimits_i \M_i \Big) \geq t
  \bigg\}& \leq d \cdot \exp \left( \frac{ -t^2/2 }{ \sigma^2 +
            R t/3 } \right).
\end{align*}
\end{lemma}

\paragraph{Linear algebraic identities.} Our proof of the
least squares approximation guarantee for LESS embeddings relies on
the following standard rank-one update formula for the matrix inverse.

\begin{lemma}[Sherman-Morrison formula]\label{l:rank-one}
For a matrix $\A \in \R^{n \times n}$ and $\u,\v \in \R^n$
such that both $\A$ and $\A + \u \v^\top$  are invertible, we have:
\begin{equation*}
    (\A + \u \v^\top)^{-1} = \A^{-1} - \frac{\A^{-1} \u \v^\top \A^{-1} }{1+\v^\top \A^{-1} \u}.
\end{equation*}
In particular, it follows that: 
\begin{equation*}
    (\A + \u \v^\top)^{-1} \u = \frac{\A^{-1} \u}{1+\v^\top \A^{-1} \u}.
\end{equation*}
\end{lemma}

\section{Hanson-Wright limit theorem: Proof of Theorem \ref{t:main}}
\label{s:main}
In this section, we prove our main result, Theorem \ref{t:main},
showing that a $k$-gaussianized sample $\xbt = \frac1{\sqrt
  k}\sum_{i=1}^kr_i\x_i$ is close in total variation distance to a
random vector that satisfies the Hanson-Wright inequality
with a small constant.

We start by using the matrix Bernstein inequality to verify how well
  the sample covariance $\frac1k\sum_{i=1}^k\x_i\x_i^\top$ can be
  bounded in terms of the true covariance $\Sigmab$. For this we
  establish the following lemma, proven in Appendix
  \ref{a:sub-exponential}. The lemma is used to show the matrix
  moment condition required by the sub-exponential matrix Bernstein
  (Lemma \ref{l:sub-exp-Bernstein}). 
  \begin{lemma}\label{l:sub-exponential}
    There is an absolute constant $C$ such that any $d$-dimensional random vector $\x$ with covariance
   $\E[\x\x^\top]=\Sigmab$, such that $\|\Sigmab^{-1/2}\x\|$ is
   $M$-sub-gaussian, satisfies the following sub-exponential matrix moment bound: 
    \begin{align*}
      \bigg\|\E\Big[\Big(\Sigmab^{-1/2}\x\x^\top\Sigmab^{-1/2} - \I\Big)^p\Big]\bigg\|\leq
\big(CM^2(p+\log M)\big)^{p-1}.
    \end{align*}
  \end{lemma}
  We can now use Lemma \ref{l:sub-exponential} in conjunction with the
  sub-exponential matrix Bernstein inequality (Lemma \ref{l:sub-exp-Bernstein}) by
  setting $\M_i = \Sigmab^{-1/2}\x_i\x_i^\top\Sigmab^{-1/2} -
  \I$. From Lemma \ref{l:sub-exponential}, we can set $R=CM^2\log M$
  and $\A_i^2=R\I$, with $\sigma^2 = kR$ (appropriately adjusting
  constant $C$), concluding that:
  \begin{align*}
    \Pr\bigg\{\lambda_{\max}\Bigg(\sum_{i=1}^k\M_i\Big)
    \geq kt\Big\}
    \leq d\cdot \exp\Big(\frac{-t^2k}{2(1+t)CM^2\log M}\Big).
  \end{align*}
 Now, define the $k\times d$ matrix $\U$ with $i$th row
 $\frac1{\sqrt k}\x_i^\top\Sigmab^{-1/2}$. The above concentration
 inequality implies that if $k\geq
 4CM^2\log^2(M/\delta)/t$ for $t\geq 1$, then $\|\U\|^2 =
 \|\U^\top\U\|\leq 1+t$ with probability $1-\delta$. We mention that, if we
 assumed $\|\Sigmab^{-1/2}\x\|\leq M$ almost surely, as in
 Remark~\ref{r:bounded}, then we can use the bounded matrix Bernstein 
 (Lemma \ref{l:bounded-Bernstein}) to obtain a slightly sharper guarantee of 
 $k\geq O(M^2\log(d/\delta)/t)$, without relying on Lemma \ref{l:sub-exponential}.

 We next define the random variable 
$\z$ and covariance matrix $\Sigmabt$ discussed in the theorem:
  \begin{align*}
    \z &= \one_{\Ec}\cdot \Sigmabt^{-1/2}\xbt + \one_{\neg\Ec}\cdot
    \alpha\V^\top\r, \qquad \Sigmabt =\frac1{\Pr\{\Ec\}}\,\E[\one_{\Ec}\xbt\xbt^\top],
    \\
  \text{for}\qquad \Ec &= \big[\|\U\|\leq L\big],\qquad
               \alpha= \max\Big\{1,\sqrt{\frac dk}\,\Big\},\qquad
                           \V=\G((\G^\top\G)^\dagger)^{1/2},
  \end{align*}
  where $\r$ is the vector of $R$-sub-gaussian random variables $r_i$
  from the definition of $\xbt$; $\G$ denotes a $k\times d$ matrix with i.i.d. Gaussian entries
  and $(\cdot)^\dagger$ is the Moore-Penrose pseudoinverse
  (so that either $\V^\top\V=\I$ or $\V\V^\top=\I$); and $\Ec$ denotes an
  event in the probability space of $\xbt$, with
$\one_{\Ec}$ being the characteristic function of $\Ec$. Here,
$L=O(1+M\log(M/\delta)/\sqrt k)$ is chosen so that
$\Pr\{\neg\Ec\}\leq\delta/L^2$.\footnote{If
  $\|\Sigmab^{-1/2}\x_i\|\leq M$ a.s., as in
  Remark~\ref{r:bounded}, then we can let $L=O(1+M\sqrt{\log(M/\delta)/k})$
  by using 
bounded matrix Bernstein (Lemma~\ref{l:bounded-Bernstein}).}
We can do this by simply adjusting the constants, because $M\log(ML/\delta) = O(M\log(M/\delta))$.
Also, we can easily make sure that $L\geq
\max\{2,\alpha\}$ (recall that $\alpha\leq 1+ \sqrt{d/k}\leq
1+M/\sqrt k$).
Note that $\z$ is an isotropic random vector since:
\begin{align*}
  \E[\z\z^\top] =
  \Sigmabt^{-1/2}\E[\one_{\Ec}\xbt\xbt^\top]\Sigmabt^{-1/2} + \Pr\{\neg\Ec\}
  \cdot \E[\alpha^2\V^\top\V]
  = \Pr\{\Ec\}\I + \Pr\{\neg\Ec\}\, \I = \I.
\end{align*}
Also, note that we can write $\z=\Ubt^\top\r$, where
\begin{align*}
  \Ubt = \one_{\Ec}\cdot \U\Sigmab^{1/2}\Sigmabt^{-1/2}
  + \one_{\neg\Ec}\cdot\alpha\V,
\end{align*}
and since $\r$ is mean zero and independent of $\Ubt$, then $\z$ is
also mean zero.

Next we show that $\Sigmabt$ approximates $\Sigmab$. First, by definition we
immediately have $\Sigmabt\preceq
\frac1{1-\delta/2}\Sigmab\preceq(1+\delta)\Sigmab$. Next, observe 
that 
we have:
\begin{align*}
      \|\Sigmab^{-\frac12}(\Sigmab - \Sigmabt) \Sigmab^{-\frac12}\|
    &= \|\Sigmab^{-\frac12}\E\big[\xbt\xbt^\top\cdot\one_{\neg
      \Ec}\big]\Sigmab^{-\frac12}\|
=\big\|\E\big[\U^\top\r\r^\top\U\cdot\one_{\neg
      \Ec}\big]\big\|
  \\
  &=\|\E[\U^\top\U\cdot\one_{\neg\Ec}]\|
\leq \E\big[\|\U\|^2\cdot \one_{\neg\Ec}\big]
    \\
    &\overset{(*)}{=}\int_0^\infty \Pr\big\{\|\U\|^2\cdot\one_{\neg\Ec}>x\big\}
      dx
    \\
    &\leq \frac{\delta}{L^2}\cdot L^2 +
      \int_{L^2}^\infty \Pr\big\{\|\U\|^2>x\big\} dx,
\end{align*}
where $(*)$ is the integral formula for the expectation of a
non-negative random variable via its cdf, and the last step uses the
observation that we have
$\Pr\{\|\U\|^2\cdot\one_{\neg\Ec}>x\}\leq\Pr\{\neg\Ec\}\leq
\delta/L^2$. To bound the integral, note that for $x\geq L^2$:
\begin{align*}
  \Pr\big\{\|\U\|^2>x\big\} \leq
  d\cdot\exp\Big(\frac{-(x-1)k}{4CM^2\log d}\Big)\leq
  (\delta/L^2)^{(x-1)/L^2}.
\end{align*}
Thus, using the formula $\int (\delta/L^2)^{x/L^2}dx =
-L^2(\delta/L^2)^{x/L^2}/\ln(L^2/\delta)$, we
obtain:
\begin{align*}
  \int_{L^2}^\infty \Pr\big\{\|\U\|^2>x\big\} dx
  &\leq   (\delta/L^2)^{-1/L^2}\int_{L^2}^\infty (\delta/L^2)^{x/L^2}
=(\delta/L^2)^{-1/L^2} \cdot L^2(\delta/L^2)/\ln(L^2/\delta)\leq \delta.
\end{align*}
We conclude that
$\|\Sigmab^{-\frac12}(\Sigmab - \Sigmabt) \Sigmab^{-\frac12}\|\leq
\delta +\delta\leq 2\delta $, obtaining:
\begin{align*}
(1-2\delta)\cdot \Sigmab\preceq  \Sigmabt\preceq (1 +
  \delta)\Sigmab,
\end{align*}
so $\Sigmabt\approx_{1+O(\delta)}\Sigmab$. Note that the above also implies that 
$\|\Sigmabt^{-1/2}\Sigmab^{1/2}\|^2 =
\|\Sigmabt^{-1/2}\Sigmab\Sigmabt^{-1/2}\|\leq
\frac1{1-2\delta}$. As a consequence we have:
\begin{align*}
  \|\Ubt\|^2
  &=
  \one_{\Ec}\cdot\|\Sigmabt^{-1/2}\Sigmab^{1/2}\U^\top\U\Sigmab^{1/2}\Sigmabt^{-1/2}\|
    +\one_{\neg\Ec}\cdot\|\alpha^2\V^\top\V\|
  \\
  &\leq
    \max\big\{\|\Sigmabt^{-1/2}\Sigmab^{1/2}\|^2\|\one_{\Ec}\U\|^2,\alpha^2\big\}
    \leq \frac{L^2}{1-2\delta}\leq 2L^2.
\end{align*}
We are now ready to establish the Hanson-Wright inequality for $\z$. 
Specifically, consider a $d\times d$ psd matrix $\B$. We study the
concentration of the quadratic form $\z^\top\B\z$ around its mean
$\tr\B$. We start with the following decomposition:
 \begin{align}
   |\z^\top\B\z - \tr\B|
   &=|\r^\top\Ubt\B\Ubt^\top\r - \tr\B|
 \leq |\r^\top\Ubt\B\Ubt^\top\r - \tr(\Ubt\B\Ubt^\top)| +
     |\tr(\Ubt\B\Ubt^\top) - \tr\B|.\label{eq:decomposition1}
 \end{align}
Since $\Ubt$ and $\r$ are independent, and we have
$\E[\r^\top\Ubt\B\Ubt^\top\r] = \tr(\Ubt\B\Ubt^\top)$, we can show
concentration for the first term in \eqref{eq:decomposition1} by using
the classical Hanson-Wright inequality (Lemma \ref{l:hanson-wright})
applied to the random vector $\r$ and the matrix $\Ubt\B\Ubt^\top$:
   \begin{align*}
   \Pr\Big\{| \r^\top\Ubt\B\Ubt^\top\r - \tr(\Ubt\B\Ubt^\top)|\geq
   t\mid \Ubt\Big\}
   &\leq
     2\exp\Big(-c\min\Big\{\frac{t^2}{R^4\|\Ubt\B\Ubt\|_F^2},
     \frac{t}{R^2\|\Ubt\B\Ubt^\top\|}\Big\}\Big)
   \\
   &\leq 2\exp\Big(-c\min\Big\{\frac{t^2}{4L^4R^4\|\B\|_F^2},
     \frac{t}{2L^2R^2\|\B\|}\Big\}\Big),
   \end{align*}
   where we used that $\|\Ubt\|^2\leq 2L^2$. Next, to show concentration
   for the second term in \eqref{eq:decomposition1}, we decompose it
   further as follows, letting
 $\Bbt=\Sigmab^{1/2}\Sigmabt^{-1/2}\B\Sigmabt^{-1/2}\Sigmab^{1/2}$:
 \begin{align}
   |\tr(\Ubt\B\Ubt^\top) - \tr\B|
   &\leq \one_{\Ec}\cdot|\tr(\U\Bbt\U^\top)- \tr\B|
   +\one_{\neg\Ec}\cdot|\alpha^2\tr\V\B\V^\top - \tr\B|\nonumber
   \\
   &\leq |\tr(\U\Bbt\U^\top)- \tr\Bbt|+|\tr\Bbt-\tr\B| +
     |\alpha^2\tr\V\B\V^\top-\tr\B|.\label{eq:decomposition2}
 \end{align}
 Note that $\tr(\U\Bbt\U)-\tr\Bbt = \frac1k\sum_{i=1}^kX_i$ for
 $X_i=\|\Bbt^{1/2}\Sigmab^{-1/2}\x_i\|^2-\tr\Bbt$. We are going to
 again use Bernstein's inequality, this time
 the scalar version, i.e., Lemma \ref{l:sub-exp-Bernstein} with
 $d=1$.\footnote{Here, again, if we have the boundedness assumption
   from Remark \ref{r:bounded}, then we can use bounded Bernstein
   (scalar version of Lemma \ref{l:bounded-Bernstein}).} Note that we have
 $\E[X_i]=0$, and also in Lemma \ref{l:sub-exp-scalar} (along the same lines as
 Lemma~\ref{l:sub-exponential}, see Appendix \ref{a:sub-exponential})
 we show that: 
 \begin{align*}
   \E[|X_i|^p]\leq (CM^2(p+\log
   M))^{p-1}\|\Bbt\|^{p-2}\|\Bbt\|_F^2.
 \end{align*}
Using Lemma \ref{l:sub-exp-Bernstein}
 with $R=O(\|\Bbt\|M^2\log M)$ and $\sigma^2 = k\cdot
 O(\|\Bbt\|_F^2M^2\log M)$, we obtain
 (for some absolute constants $c,c'>0$): 
  \begin{align*}
   \Pr\Big\{ |\tr(\U\Bbt\U^\top)- \tr\Bbt|\geq t\Big\}
   &\leq
   2\exp\Big(-c\min\Big\{\frac{t^2k}{\|\Bbt\|_F^2M^2\log
     M},\frac{tk}{\|\Bbt\|M^2\log M}\Big\}\Big)
   \\
   &\leq
     2\exp\Big(-c'\min\Big\{\frac{t^2}{L^2\|\B\|_F^2},\frac{t}{L^2\|\B\|}\Big\}\Big),
 \end{align*}
where we used the facts that $\|\Bbt\|\leq 2\|\B\|$ and $L^2\geq
M^2\log(M)/k$. We now turn to the second term in
\eqref{eq:decomposition2}. Note that this is in fact a deterministic
quantity which can be bounded as follows: $|\tr\Bbt-\tr\B|\leq 4\delta
\cdot\tr\B$. Now, assuming that $\delta\leq 1/(4d)$ (which is
w.l.o.g. by adjusting the constants), we have that for any $0\leq t\leq 4\delta \tr\B$:
\begin{align*}
  \min\Big\{\frac{t^2}{\|\B\|_F^2},\frac t{\|\B\|}\Big\}
  &\leq
    \min\Big\{\frac{(\tr\B)^2}{d^2\|\B\|_F^2},\frac{\tr\B}{d\|\B\|}\Big\}
 \leq
    \min\Big\{\frac{\|\B\|^2}{\|\B\|_F^2},\frac{\|\B\|}{\|\B\|}\Big\}\leq 1,
\end{align*}
where we used that $\tr\B/d\leq \|\B\|$. This means that $\Pr\{\|\tr\Bbt-\tr\B|\geq
t\}\leq \one_{[t\leq 4\delta\cdot\tr\B]}\leq
2\exp(-c\min\{\frac{t^2}{\|\B\|_F^2},\frac{t}{\|\B\|}\})$,
so the desired concentration inequality trivially holds. Finally, it remains to establish the concentration for
$|\alpha^2\tr(\V\B\V^\top)-\tr\B|$, the final term in
\eqref{eq:decomposition2}. If $k\geq d$, then $\V^\top\V=\I$ and
$\alpha^2=1$, so the term is $0$. Now, suppose that $k<d$. In this
case, $\V^\top\V$ is a projection onto a uniformly random
$k$-dimensional subspace of $\R^d$. Here, we will use a simple form of
the Johnson-Lindenstrauss lemma for uniformly random projections
\citep[Lemma 5.3.2,][]{vershynin2018high}, which states that for any
fixed unit vector $\b\in\R^d$:
\begin{align*}
  \Pr\Big\{\big|\|\alpha\V\b\|-1\big|\geq \epsilon\Big\}
  \leq2\exp(-c\epsilon^2k). 
\end{align*}
Let $\B=\sum_{i=1}^d\lambda_i\b_i\b_i^\top$ be the eigendecomposition of
$\B$. The Johnson-Lindenstrauss concentration lemma implies that with
probability $1-2d\exp(-c\epsilon^2k)$ we have
$\max_i|\|\alpha\V\b_i\|-1|\leq\epsilon$, which in turn implies that:
\begin{align*}
  |\alpha^2\tr(\V\B\V^\top)-\tr\B|\leq
  \sum_{i=1}^d\lambda_i\big|\|\alpha\V\b_i\|^2-1\big|
  \leq \tr\B\cdot \max_i \big|\|\alpha\V\b_i\|^2-1\big|\leq \epsilon(\epsilon+2)\cdot\tr\B.
\end{align*}
Setting $t=\epsilon(\epsilon+2)\tr\B$ and solving for $\epsilon$, we
convert this to a concentration inequality:
\begin{align*}
  \Pr\Big\{|\alpha^2\tr(\V\B\V^\top)-\tr\B|\geq t\Big\}
  &\leq
    2d\exp\Big(-c\min\Big\{\frac{t^2k}{\tr\B},\frac{tk}{\tr\B}\Big\}\Big)
  \\
  &\leq 2\exp\Big(-c\min\Big\{\frac{t^2k}{\|\B\|d\log
    d},\frac{tk}{\|\B\|d\log d}\Big\}\Big)
  \\
  &\leq 2\exp\Big(-c\min\Big\{\frac{t^2}{L^2\|\B\|_F^2},\frac{t}{L^2\|\B\|}\Big\}\Big).
\end{align*}
Putting everything together, we combine the four inequalities to obtain that for any psd matrix $\B$:
\begin{align*}
  \Pr\big\{\z^\top\B\z-\tr\B|\geq t\big\}\leq 4\cdot
  2\exp\Big(-c\min\Big\{\frac{t^2}{(LR)^4\|\B\|_F^2},\frac{t}{(LR)^2\|\B\|}\Big\}\Big).
\end{align*}
We can easily extend this to arbitrary matrices $\B$. First, observe
that if $\B$ is symmetric, then we can write it as $\B=\B_+ - \B_-$
where $\B_+$ and $\B_-$ are psd, and it
suffices to apply the concentration inequality to both $\B_+$ and
$\B_-$. Next, consider an arbitrary $\B$. Then, it suffices to apply
the result to $(\B+\B^\top)/2$, which is symmetric.


\section{Sketched least squares: Proof of Theorem \ref{t:main-app}}
\label{a:main-app}
In this section, we prove Theorem \ref{t:main-app}, giving a precise
estimate for the expected approximation error for sketched least
squares with LESS embeddings.
Recall that  we are given an $N\times d$ data matrix $\A$ and an
$N$-dimensional vector $\b$. We let $L(\w)=\|\A\w-\b\|^2$ and $\w^* =
\argmin_\w L(\w)$. Given an $n\times N$ LESS embedding
matrix $\S$, we define the sketched least squares estimator as:
  \begin{align*}
   \wbh = \argmin_{\w\in\mathcal R^d}\|\S(\A\w-\b)\|^2,
  \end{align*}
  where $\mathcal R=[-D,D]$ is only used to protect against the corner
  cases where $\S\A$ is very ill-conditioned, which arise with a small but non-zero probability when
  computing the exact expectation of the loss.
In the proof, we will use our reduction between a LESS embedding and a
sub-gaussian design to define an event that holds with a
high $1-\delta$ probability (Lemma~\ref{l:event}). This event ensures certain
regularity conditions on the sketch, such as near-uniformity of the
leverage scores, and avoids the aforementioned corner cases. We then proceed to analyze the conditional expectation,
conditioned on the high-probability event. To obtain the final
unconditional expectation, we must choose $\delta$ sufficiently small
to absorb the worst-case loss that might occur if the high-probability
event fails. This is where we incur a dependence on the range diameter $D$,
which is in general unavoidable when computing the expected loss, but it only affects the
logarithmic factors in the result.
  
Let us briefly comment on those logarithmic dependencies. If we let the LESS embedding have $d$ non-zeros per row, as
recommended by \cite{less-embeddings}, then our reduction to a
sub-gaussian sketch (which is used in Lemma \ref{l:event}) incurs a
logarithmic dependence in the Hanson-Wright constant, i.e.,
$K=O(\sqrt{\log(nd/\delta)})$, and also, the lower bound on the sketch size
required for the claim becomes $n\geq O(d\log(d/\delta))$. This also
results in a polylogarithmic factor in the error $\epsilon$. We note
that these logarithmic terms can be avoided if we slightly
increase the density of the sketch, from $d$ to $d\log(nd/\delta)$
non-zeros per row. Then, the result can be established for
$n\geq O(d)$.
  
Instead of computing the expectation of $L(\wbh)$ directly, we will
define an unbiased estimate of the expected loss via leave-one-out
cross-validation (CV):
  \begin{align*}
    L_{\CV} = \sum_{i=1}^n
    \big(\s_i^\top\A\wbh^{-i}
    -\s_i^\top\b\big)^2,\qquad\text{for}\quad\wbh^{-i} =
    \argmin_{\w\in\mathcal R^d}\|\S_{-i}(\A\w-\b)\|^2,
  \end{align*}
  where $\s_i^\top$\! is the $i$th row of $\S$ and $\S_{-i}$ denotes $\S$ without the $i$th row. We use the following
  standard ``shortcut formula'' for the CV estimate, which holds as
  long as the scalar range $\mathcal R=[-D,D]$ is large enough so that
  each $\wbh^{-i}$ is the same as the unconstrained least squares solution (we include
  the proof in Appendix \ref{a:cv}). 
  \begin{lemma}\label{l:cv}
    Suppose that $(\A^\top\S_{-i}\S_{-i}\A)^{-1}\preceq
    O(1)\cdot(\A^\top\A)^{-1}$ for all $i$. Then, for sufficiently
    large $D$, the sketched estimator is given by
    $\wbh=(\A^\top\S^\top\S\A)^{-1}\A^\top\S^\top\S\b$, and the leave-one-out cross-validation loss can be computed as:
  \begin{align*}
    L_{\CV} =\sum_{i=1}^n\bigg(\frac{\s_i^\top\A\wbh  -\s_i^\top\b}{1-\ell_i(\S\A)}\bigg)^2,
  \end{align*}
where $\ell_i(\S\A)= \s_i^\top\A
  (\A^\top\S^\top\S\A)^{-1}\A^\top\s_i$ is the $i$th
  leverage score of the sketch.
\end{lemma}
Crucially, the CV loss is an unbiased estimate of the expected loss
  based on the sketch of size $n-1$:
  \begin{align*}
    \E[L_{\CV}]= \E[L(\wbh^{-n})],
  \end{align*}
  Thus, for the rest of the proof, we will actually estimate
  $\E[L(\wbh^{-n})]$ rather than $\E[L(\wbh)]$, but the difference
  between the estimates is sufficiently small to be absorbed by
  $\epsilon$ since $\frac {d}{n-d} \approx_{1+\epsilon} \frac{d}{n-d-1}$.

  A key part of the analysis is showing that the leverage scores of
  the sketched matrix are nearly uniform with high probability. This
  is established in the following lemma.
  \begin{lemma}\label{l:event}
    With probability $1-\delta$, for every $i\in[n]$ we have:
    \begin{align*}
   (\A^\top\S_{-i}^\top\S_{-i}\A)^{-1}\preceq O(1)\cdot(\A^\top\A)^{-1}\qquad\text{and}\qquad   \s_i^\top\A(\A^\top\S_{-i}^\top\S_{-i}\A)^{-1}\A^\top\s_i
     \, \approx_{1+\epsilon} \frac d{n-d},
    \end{align*}
  for $\epsilon = \tilde O(1/\sqrt d)$.  In particular, this implies that
    $\ell_i(\S\A)\approx_{1+2\epsilon}\frac dn$ for all $i$.
  \end{lemma}
  We now define a high probability event $\Ec$ that ensures uniformity of the leverage scores:
  \begin{align*}
    \Ec\ =\
    \Big[(\A^\top\S_{-i}^\top\S_{-i}\A)^{-1}\preceq O(1)\cdot(\A^\top\A)^{-1},\quad
    \s_i^\top\A(\A^\top\S_{-i}^\top\S_{-i}\A)^{-1}\A^\top\s_i
    \,\approx_{1+\epsilon} \frac d{n-d}\quad\forall_i\Big].
  \end{align*}
In particular, the event implies that $1-\ell_i(\S\A)\approx_{1+\epsilon'}1-\frac
dm$ for $\epsilon'=\tilde O(\sqrt d/n)$,
which allows us to simplify the CV loss by approximating all the
leverage scores with a uniform estimate. In what follows, we will use
$\E_{\Ec}$ to denote expectation conditioned on the event $\Ec$:
\begin{align*}
\E_{\Ec}[L_{\CV}] &\approx_{1+\epsilon'} \E_{\Ec}\bigg[\sum_{i=1}^n\bigg(\frac{\s_i^\top\A\wbh
    -\s_i^\top\b}{1-d/n}\bigg)^2\bigg]
=\Big(\frac{n}{n-d}\Big)^2\cdot
  \E_{\Ec}\big[\|\S(\A\wbh-\b)\|^2\big].
\end{align*}
Thus, it remains to approximate
$\E_{\Ec}\big[\min_\w\|\S(\A\w-\b)\|^2\big]$ in terms of
$L(\w^*)=\min_\w\|\A\w-\b\|^2$. Here, we are going to take advantage
of the similar form of the two quantities, in that they can be both
written as squared distance between a vector and a subspace. In the case of
$L(\w^*)$, it is the squared distance between $\b$ and the column-span
of $\A$, and it can be written as:
\begin{align*}
  L(\w^*) = \b^\top(\I-\H)\b,\qquad\text{for}
  \quad\H = \A(\A^\top\A)^{-1}\A,
\end{align*}
where $\H$ is the so-called hat matrix, an orthogonal projection onto
the column-span of $\A$. In the sketched version, we get a squared
distance between $\S\b$ and the column-span of $\S\A$, which can be
written similarly, using a sketched version of the hat matrix:
\begin{align*}
  \min_\w\|\S(\A\w-\b)\|^2 =
  \b^\top\S^\top(\I-\Hbh)\S\b,\qquad\text{for}
  \quad\Hbh = \S\A(\A^\top\S^\top\S\A)^{-1}\A^\top\S^\top.
\end{align*}
Thus, both quantities are quadratic forms applied to the vector
$\b$. We relate these two quadratic forms in expectation through the following result.
\begin{lemma}\label{l:hatmatrix}
  The sketched hat matrix $\Hbh$ satisfies the following approximate
  expectation formula:
  \begin{align*}
    \E_{\Ec}\big[\S^\top(\I-\Hbh)\S\big] \approx_{1+\epsilon'} \Big(1-\frac
    dn\Big)\cdot (\I-\H),\qquad\text{for}\quad \epsilon' = \tilde O(\sqrt d/n).
  \end{align*}
\end{lemma}
\begin{remark}
Here, $\approx_{1+\epsilon}$ refers to upper/lower bounds in terms of the
positive semidefinite ordering. This result extends an exact
expectation formula that holds when $\S$ is a Gaussian embedding.
\end{remark}
Putting everything together, we obtain the following approximation of
the expected CV loss conditioned on $\Ec$:
\begin{align*}
  \E_{\Ec}[L_\CV]
  &\approx_{1+\epsilon'}
  \Big(\frac{n}{n-d}\Big)^2\b^\top\E_{\Ec}\big[\S^\top(\I-\Hbh)\S\big]\b
  \\
  &\approx_{1+\epsilon'}\Big(\frac{n}{n-d}\Big)^2\Big(1-\frac dn\Big)\cdot L(\w^*) =
  \frac{n}{n-d}\cdot L(\w^*).
\end{align*}
From this it follows that
$$\E_{\Ec}[L_{\CV}]-L(\w^*)
\approx_{1+\epsilon} \frac {d}{n-d}\cdot
L(\w^*)\approx_{1+\epsilon}\frac{d}{n-d-1}\cdot L(\w^*),$$
for $\epsilon=\tilde O(1/\sqrt d)$. Note that for this we have to
switch from $\epsilon'=\tilde O(\sqrt d/n)$ to $\epsilon=\tilde O(1/\sqrt d)$, because
the right-hand side gets scaled by $d/n$. Finally, we convert from
conditional expectation as follows:
\begin{align*}
  (1-\delta)\E_{\Ec}[L_{\CV}]
  \leq \E[L_{\CV}]
  &\leq \E_{\Ec}[L_{\CV}] +
    \delta \sup_{\S}L_{\CV}
  \\
  &\leq \E_{\Ec}[L_{\CV}] + \delta\cdot
  n\cdot\sup_{\S}\|\s_n\|^2\|\A\wbh^{-n}-\b\|^2
  \\
  &\leq \E_{\Ec}[L_{\CV}] + \delta\cdot \poly(N,D),
\end{align*}
where in the last step, to avoid the corner cases where sketch $\S\A$
is extremely ill-conditioned, we used that $\wbh$, as well as $\A$ and
$\b$, have entries in $[-D,D]$, and that, by our assumption on the
sparsifying distribution $p$, a LESS
embedding matrix $\S$ has entries uniformly bounded by $O(\sqrt{Nk})$.
Choosing sufficiently small $\delta\leq
\epsilon L(\w^*)/\poly(N,D)$, we can obsorb that bound
into a logarithmic factor that appears in the upper bound on
$\epsilon$ (if we use $d$ non-zeros per row in LESS). Alternatively, we can increase the
number of non-zeros per row in LESS to $d\log(nd/\delta)$, in which
case, we avoid the dependence of $\epsilon$ on $D$.
We note that another simple way to impose a strict bound on the
worst-case loss is to use a small $\ell_2$ regularizer when computing the
sketched least squares, ensuring
that the sketched problem is never ill-conditioned. Extending our
analysis to the regularized least squares setting is an interesting
direction for future work.

\subsection{Proof of Lemma \ref{l:hatmatrix}}
Let $\U=\A(\A^\top\A)^{-1/2}$ and note that $\U$ is an $N\times d$
matrix that satisfies $\U^\top\U=\I$.  We can rewrite both $\H$ and
$\Hbh$ by replacing $\A$ with 
  $\U$, namely: $\H=\U\U^\top$ and $\Hbh =
  \S\U(\U^\top\S^\top\S\U)^{-1}\U^\top\S^\top$.
 To show the claim, it suffices to show that for every vector
 $\v\in\R^N$, the following approximation holds:
 \begin{align}
   \E_{\Ec}[\v^\top\S^\top(\I-\Hbh)\S\v]\,\approx_{1+\epsilon'}\Big(1-\frac
   dn\Big)\,\v^\top(\I-\H)\v.\label{eq:hatmatrix}
   \end{align}
First, if $\v$ lies in the column-span of $\U$, i.e.,
$\v=\U\x$ for some $\x$, then clearly the right-hand side is
zero, and also: 
\begin{align*}
  \v^\top\S^\top(\I-\Hbh)\S\v = \x^\top\U^\top\S^\top\S\U\x -
  \x^\top\U^\top\S^\top\Hbh\S\U\x
  =0 = \Big(1-\frac dn\Big)\,\v^\top(\I-\H)\v.
\end{align*}
Thus, it suffices to show \eqref{eq:hatmatrix} for $\v$ orthogonal to
the column-span of $\U$, i.e., one that satisfies
$\U^\top\v=\zero$. In that case, 
\begin{align*}
  \Big|\v^\top\E_{\Ec}[\S^\top(\I-\Hbh)\S]\v - \Big(1-\frac
  dn\Big)\v^\top(\I-\H)\v\Big| = \Big|\v^\top\E_{\Ec}[\S^\top\S]\v - \v^\top\E_{\Ec}[\S^\top\Hbh\S]\v
  -\Big(1-\frac dn\Big)\Big|.
\end{align*}
We now focus on analyzing $\E_{\Ec}[\S^\top\Hbh\S]$. To that end,
let us use the shorthands $\Q=(\gamma\U^\top\S^\top\S\U)^{-1}$ and
$\Q_{-i} = (\gamma\U^\top\S_{-i}^\top\S_{-i}\U)^{-1}$, where $\gamma =
\frac{n}{n-d}$, so that $\Hbh =
\S\U\gamma\Q\U^\top\S^\top$. We also use the formula $\s_i^\top\U\Q =
\s_i^\top\U\Q_{-i}/\gamma_i$, where $\gamma_i = 1
+\gamma\s_i^\top\U\Q_{-i}\U^\top\s_i$, which is a consequence of the
Sherman-Morrison formula (Lemma \ref{l:rank-one}). Recall that,
from Lemma~\ref{l:event}, the event $\Ec$
implies that $\gamma_i\approx\gamma$. Also, let us use a simplifying
shorthand of $\sbt_i=\sqrt n\,\s_i$, so that
$\E[\sbt_i\sbt_i^\top]=\I$. We can now rewrite the expectation as
follows (where $i$ is any fixed index):
\begin{align*}
\E_{\Ec}[\S^\top\Hbh\S]
  &= \E_{\Ec}[\S^\top\S\U\gamma\Q\U^\top\S^\top\S]
= \E_{\Ec}[\sbt_i\sbt_i^\top\U\gamma\Q\U^\top\S^\top\S]
=\E_{\Ec}\Big[\frac{\gamma}{\gamma_i}\sbt_i\sbt_i^\top\U\Q_{-i}\U^\top\S^\top\S\Big]
      \\
    &=
      \frac1n\E_{\Ec}\Big[\frac{\gamma}{\gamma_i}\sbt_i\sbt_i^\top\U\Q_{-i}\U^\top\sbt_i\sbt_i^\top]
      +
      \E_{\Ec}\Big[\frac{\gamma}{\gamma_i}\sbt_i\sbt_i^\top\U\Q_{-i}\U^\top\S_{-i}^\top\S_{-i}\Big]
  \\
  &=\E_{\Ec}\Big[\frac{\gamma_i-1}{\gamma_i}\sbt_i\sbt_i^\top]
      +
      \E_{\Ec}\Big[\frac{\gamma}{\gamma_i}\sbt_i\sbt_i^\top\U\Q_{-i}\U^\top\S_{-i}^\top\S_{-i}\Big]
    \\
    &=\E_{\Ec}\Big[\frac{\gamma-1}{\gamma}\cdot\sbt_i\sbt_i^\top\Big]
      +\E_{\Ec}\big[\U\Q_{-i}\U^\top\S_{-i}^\top\S_{-i}\big]
    \\
    &\quad+
      \E_{\Ec}\Big[\Big(\frac1{\gamma}-\frac1{\gamma_i}\Big) \sbt_i\sbt_i^\top\Big]
      +\E_{\Ec}\Big[\Big(\frac{\gamma}{\gamma_i}-1\Big)\sbt_i\sbt_i^\top\U\Q_{-i}\U^\top\S_{-i}^\top\S_{-i}\Big]=:\T_1+\T_2+\T_3+\T_4.
\end{align*}
Thus, it follows that:
\begin{align*}
  \E_{\Ec}\big[\S^\top(\I-\Hbh)\S\big]
  &= \E_{\Ec}\big[\sbt_i\sbt_i^\top] -
  \E_{\Ec}\Big[\frac{\gamma-1}{\gamma}\sbt_i\sbt_i^\top\Big] -
  (\T_2+\T_3+\T_4)
  \\
  &=\Big(1-\frac dn\Big)\E_{\Ec}[\sbt_i\sbt_i^\top] - (\T_2+\T_3+\T_4).
\end{align*}
In the following steps, we will show that
$\E_{\Ec}[\sbt_i\sbt_i^\top]\approx \I$ and that
$\v^\top\T_k\v\approx 0$ for $k\in\{2,3,4\}$. First, denoting $\delta
= \Pr(\neg\Ec)$, we have:
\begin{align*}
\|\E_{\Ec}[\sbt_i\sbt_i^\top] - \I\| =
  \frac1{1-\delta}\|\E[\one_{\neg\Ec}\sbt_i\sbt_i]\|\leq
  \frac{\delta}{1-\delta}\cdot O(Nk),
\end{align*}
where we used that $\|\sbt_i\|^2\leq O(Nk)$.
Next, observe that for $\v$ orthogonal to the column-span of $\U$, we
have $\v^\top\T_2\v=\v^\top\U\Q_{-i}\U^\top\S_{-i}^\top\S_{-i}\v=0$
because $\v^\top\U=\zero$. Furthermore, using Lemma~\ref{l:event},
conditioned on $\Ec$ we
know that $|\gamma_i-\gamma|=\tilde O(\sqrt d/n)$, which allows us to control
the remaining terms. In the following, we also use the fact that
$\E_{\Ec}[X]\leq \frac1{1-\delta}\E[X]$ for any non-negative $X$:
\begin{align*}
  \big|\v^\top\T_3\v\big|
  &\leq
    \E_{\Ec}\Big[\frac{|\gamma_i-\gamma|}{\gamma\gamma_i}(\v^\top\sbt_i)^2\Big]
\leq \tilde O(\sqrt d/n)\cdot \E_{\Ec}[(\v^\top\sbt_i)^2]
  \\
  &\leq \tilde O(\sqrt d/n) \cdot \v^\top\E[\sbt_i\sbt_i]\v = \tilde O(\sqrt d/n)\cdot \|\v\|^2.
\end{align*}
Similarly, we bound the last term, using the Cauchy-Schwartz inequality:
\begin{align*}
  \big|\v^\top\T_4\v\big|
  &\leq \tilde O(\sqrt d/n)\cdot \E_{\Ec}\big[|\v^\top\sbt_i|\cdot
    |\sbt_i^\top\U\Q_{-i}\U^\top\S_{-i}^\top\S_{-i}\v|\big]
  \\
  &\leq \tilde O(\sqrt d/n)\cdot \sqrt{\E_{\Ec}\big[(\v^\top\sbt_i)^2\big]}\cdot
    \sqrt{\E_{\Ec}\big[(\sbt_i^\top\U\Q_{-i}\U^\top\S_{-i}^\top\S_{-i}\v)^2\big]}.
\end{align*}
As before, we have $\E_{\Ec}[(\v^\top\sbt_i)^2]\leq O(1)\cdot
\|\v\|^2$. For the last term in the product, we use a LESS embedding property of
the random vector vector $\U^\top\sbt_i$ that was shown in
\cite[Lemma 28]{less-embeddings}, which they called the Bai-Silverstein
property. In particular, this property implies that for any vector $\v$, we have
$\E[(\sbt_i^\top\U\v)^2]\leq O(1)\cdot \|\v\|^2$ (this property also holds
for any sub-gaussian vector, so it could be easily inferred from
our Theorem~\ref{c:less-embeddings}, after conditioning on a high-probability
event). We also use that the event $\Ec$ implies that $\|\Q_{-i}\| =
O(1)$. It follows that:
\begin{align*}
  \E_{\Ec}\big[(\sbt_i^\top\U\Q_{-i}\U^\top\S_{-i}^\top\S_{-i}\v)^2\big]
  &\leq
\frac1{1-\delta}
    \E\big[(\sbt_i^\top\U\Q_{-i}\U^\top\S_{-i}^\top\S_{-i}\v)^2\mid
    \|\Q_{-i}\|=O(1)\big]
    \\
  &\leq O(1)\cdot\E\big[\|\Q_{-i}\U^\top\S_{-i}^\top\S_{-i}\v\|^2\mid\|\Q_{-i}\|=O(1)\big]
  \\
  &\leq O(1)\cdot \E\big[\|\U^\top\S_{-i}^\top\S_{-i}\v\|^2\big]
    \\
  &=O(1)\cdot \E\bigg[\Big\|\frac1n\sum_{j\neq i}\U^\top\sbt_j\sbt_j^\top\v\Big\|^2\bigg]
\leq O(1)\cdot
    \frac1{n}\E\big[\|\U^\top\sbt_j\sbt_j^\top\v\|^2\big],
\end{align*}
where in the last step we used that $\sbt_j$ are independent, and
$\U^\top\v=\zero$. To bound $\E[\|\U^\top\sbt_j\sbt_j^\top\v\|^2]$, we
will use the definition of a $p$-sparsified embedding
(Definition~\ref{d:leverage-score-sparsifier}), where we let $I_i$ be
the $k$ random indices that indicate 
the non-zero entries of
$\sbt_j=\sum_{i=1}^k\frac{r_i\e_{I_i}}{\sqrt{kp_{I_i}}}$ and $r_i$'s are
the random signs.
\begin{align*}
  \E\big[\|\U^\top\sbt_j\sbt_j^\top\v\|^2\big]
  \leq k\,\E\bigg[\frac{\|\U^\top\e_{I_1}\|^2(\e_{I_1}^\top\v)^2}{k^2
  p_{I_1}^2}\bigg]
  +
  k^2\,\E\bigg[\frac{\|\U^\top\e_{I_1}\|^2}{kp_{I_1}}\bigg]\E\bigg[\frac{(\e_{I_2}^\top\v)^2}{kp_{I_2}}\bigg]
  \leq \big(Cd/k + Cd\big)\cdot \|\v\|^2,
  \end{align*}
where, to eliminate the cross-terms when expanding
the square, we used independence of $I_i$'s and that $\U^\top\v=\zero$. In
the second step, we used that $p_{I_i}\geq \|\U^\top\e_{I_i}\|^2/(Cd)$.
Thus, $|\v^\top\T_4\v|\leq \tilde
O(d/n^{1.5})\|\v\|^2=\tilde O(\sqrt d/n) \|\v\|^2$. We
conclude that for sufficiently small $\delta$, claim
\eqref{eq:hatmatrix} is satisfied with $\epsilon=\tilde O(\sqrt d/ n)$, which
completes the proof.

  \subsection{Proof of Lemma \ref{l:event}}
  \label{a:ls}


We note that the first claim, i.e, that
$(\A^\top\S_{-i}\S_{-i}\A)^{-1}\preceq O(1)\cdot(\A^\top\A)^{-1}$
follows easily for LESS embeddings with $d$ non-zeros per row of size $n\geq O(d\log(d/\delta))$, by
relying on the subspace embedding property, shown in
\cite[Lemma 12]{less-embeddings}, which essentially follows from the
matrix Bernstein inequality (Lemma \ref{l:bounded-Bernstein}). An alternative way of showing this
claim is to rely on the subspace
embedding guarantee for sub-gaussian embeddings, and then convert that
to a result for LESS embeddings via our
Theorem~\ref{c:less-embeddings}. This approach is given in Corollary~\ref{c:low-distortion}.
Interestingly, this leads to a slightly 
different trade-off in terms of the logarithmic factors. 
Namely, if we consider a LESS embedding with $d\log(nd/\delta)$
non-zeros per row, slightly more than recommended by
\cite{less-embeddings}, then the subspace embedding
guarantee holds for $n\geq O(d)$ samples, rather than $n\geq O(d\log d)$.
Such a guarantee is new for LESS embeddings, since it cannot be obtained via the
matrix Bernstein inequality.
  
We now focus on establishing the second part of the claim, i.e., the
uniformity of the leverage scores, which is where the Hanson-Wright
inequality is essential. Since we want to show a statement that holds with high probability, we
can use Theorem~\ref{c:less-embeddings} to show this claim for an
$n\times d$ sub-gaussian design $\X = \Z\Sigmab^{1/2}$ with $\Z$ consisting
of mean zero isotropic rows that satisfy the
Hanson-Wright inequality \eqref{eq:hanson-wright} with constant
$\tilde O(1)$, and then convert
the claim to LESS embeddings. In this case, our goal is to show that
for $n\geq \tilde O(d)$, with probability $1-\delta$:
\begin{align*}
\x_i^\top\big(\X_{-i}^\top\X_{-i}\big)^{-1}\x_i\approx_{1+\epsilon}\frac{d}{n-d}.
\end{align*}
Note that, by Hanson-Wright, with high probability we immediately have that
$\x_i^\top\big(\X_{-i}^\top\X_{-i}\big)^{-1}\x_i\approx
\tr\,\Sigmab(\X_{-i}^\top\X_{-i})^{-1}$ (this will be
formalized later on), so the main challenge is in showing that the
trace functional is also concentrated. Below, we provide a slightly
more general result, regarding the concentration properties of a
certain class of trace functionals for a sub-gaussian design, which
may be of independent interest. (Below, we use $a\lesssim b$ to
denote $a\leq cb$ for an absolute constant $c$).
  \begin{lemma}\label{l:trace-functionals}
    Suppose that $\X$ consists of rows $\x_1^\top,\x_2^\top,...,\x_n^\top$
    for  $n\geq \tilde O(d)$, that are i.i.d. sampled from a $d$-variate distribution with
    covariance $\E[\x_i\x_i^\top]=\Sigmab$ such that
    $\Sigmab^{-1/2}\x_i$ satisfies the Hanson-Wright inequality
    \eqref{eq:hanson-wright} with constant $K=\tilde O(1)$. Consider 
    $F(\X)=\tr\,\B(\frac1n\X^\top\X+\C)^{-1}$ for some psd matrices $\B$ and
    $\C$. Then, there is a scalar $\tilde F$ such that with probability $1-\delta$:
    \begin{align*}
      |F(\X) - \tilde F| \leq C\|\M\|\cdot \frac{\mathrm{r}(\M) +
      K^4}{\sqrt n}\log1/\delta,
    \end{align*}
    where $\M = \B^{1/2}(\Sigmab+\C)^{-1}\B^{1/2}$, $\mathrm{r}(\M)=\tr(\M)/\|\M\|$, and $C$ is an
    absolute constant.
  \end{lemma}
  \begin{proof}
    Note that $\C$ may not be positive definite (in fact, of our
    primary interest is when $\C=\zero$), so function $F(\X)$ may not
    be well defined everywhere because of the inverse. In particular,
    it may not have a bounded expectation. So, to define our scalar
    $\tilde F$, we first construct a high-probability event $\Ec$ which
    will ensure that $F(\X)$ is not only well-defined, but also
    sufficiently well behaved. Without loss of generality, assume that
    $n$ is even. Now, define:
    \begin{align*}
      \Ec_1 = \Big[\frac2n\sum_{i=1}^{n/2}\x_i\x_i^\top\succeq
      \frac1C\Sigmab\Big],\qquad
      \Ec_2 = \Big[\frac2n\sum_{i=n/2+1}^{n}\x_i\x_i^\top\succeq
      \frac1C\Sigmab\Big],
    \end{align*}
    with $\Ec = \Ec_1\wedge\Ec_2$. Since $n\geq\tilde O(d)$, by
    relying on the subspace embedding property (as discussed at the
    beginning of the section), 
    each of the events $\Ec_1$ and $\Ec_2$ holds with 
    probability at least $1-\delta$. We will use $\tilde
    F=\E_{\Ec}[F(\X)]$ as our scalar. Note that, conditioned on $\Ec$,
    the function $F$ satisfies $0\leq F(\X)\leq C\,\tr\,\M$, so in
    particular the conditional expectation is well-defined.

    Next, we proceed with the following decomposition of the trace
    (following standard literature in random matrix theory). Let
    $\Q=(\frac1n\X^\top\X+\C)^{-1}$ and
    $\Q_{-i}=(\frac1n\X_{-i}^\top\X_{-i}+\C)^{-1}$ where $\X_{-i}$
    denotes $\X$ without the $i$ row. Also, let $\E_i$ denote
    expectation conditioned on $\x_1,...,\x_i$. We have, conditioned
    on $\Ec$ that:
    \begin{align*}
      F(\X) - \tilde F
      &= \E_n[\tr\B\Q\mid \Ec] - \E_0[\tr\B\Q\mid\Ec]
= \sum_{i=1}^n(\E_i - \E_{i-1})[\tr\B\Q\mid\Ec]
      \\
      & = \sum_{i=1}^n(\E_i - \E_{i-1})[\tr\B(\Q-\Q_{-i})\mid\Ec] +
        (\E_i - \E_{i-1})[\tr\B\Q_{-i}\mid\Ec].
    \end{align*}
    We proceed to bound each of the two terms. The latter can be
    bounded straightforwardly. Note that one of $\Ec_1$ or $\Ec_2$ is
    independent of $\x_i$. Without loss of generality suppose that it
    is $\Ec_1$. Then, we have:
    \begin{align*}
      \big|(\E_i - \E_{i-1})[\tr\B\Q_{-i}\mid\Ec]\big| \leq
      \frac{\delta}{1-\delta}\cdot
      \big|(\E_i-\E_{i-1})[\tr\B\Q_{-i}\mid\Ec_1\wedge\neg\Ec_2]\big|
      \lesssim \delta\cdot\tr\,\M.
    \end{align*}
    To bound the first term, we use the Sherman-Morrison rank-one
    update formula (Lemma \ref{l:rank-one}), observing that:
    \begin{align*}
      |\tr\B(\Q-\Q_{-i})| =
      \frac{\frac1n\tr\,\B\Q_{-i}\x_i\x_i^\top\Q_{-i}}{1+\frac1n\x_i^\top\Q_{-i}\x_i}
      \leq \frac1n \x_i^\top\Q_{-i}\B\Q_{-i}\x_i = \frac1n\z_i^\top\tilde\M_i\z_i,
    \end{align*}
    where we let $\z_i=\Sigmab^{-1/2}\x_i$ and
    $\tilde\M_i=\Sigmab^{1/2}\Q_{-i}\B\Q_{-i}\Sigmab^{1/2}$. Now,
since $\z_i$ satisfies Hanson-Wright inequality with constant $K$, it
also satisfies Euclidean concentration with constant $O(K^2)$ (see
Proposition \ref{p:euclidean}), and in
particular, letting $Z_i=(\z_i^\top\tilde\M_i\z_i)^{1/2}$ and conditioning on $\tilde\M_i$ we have:
$\|Z_i-\E[Z_i\mid\tilde\M_i]\|_{\psi_2}\lesssim K^2\|\tilde\M_i\|^{1/2}$. Thus,
defining $M_i=\E[Z_i\mid\tilde\M_i]\leq (\tr\tilde\M_i)^{1/2}$, we have:
    \begin{align*}
      \E\big[|\z_i^\top\tilde\M_i\z_i|^p\mid\Ec\big]^{1/p}
      & \leq\frac1{1-\delta}
        \E\big[|\z_i^\top\tilde\M_i\z_i|^p\mid\Ec_1\big]^{1/p}
=\frac1{1-\delta} \E\big[(M_i -
        M_i+Z_i)^{2p}\mid\Ec_1]^{1/p}
      \\
      &\lesssim \E\big[\E[M_i^{2p}\mid\tilde\M_i]\mid\Ec_1\big]^{1/p}
        +
        \E\big[\E[|Z_i-M_i|^{2p}\mid\tilde\M_i]\mid\Ec_1\big]^{1/p}
      \\
      &\lesssim \E\big[(\tr\tilde\M_i)^p\mid\Ec_1\big]^{1/p}
        +
        \E\big[(\sqrt pK^2\|\tilde\M_i\|^{1/2})^{2p}\mid\Ec_1\big]^{1/p}
      \\
      &\lesssim \tr\M + pK^4\|\M\|.
    \end{align*}
    where the last step follows because conditioned on $\Ec_1$ we have
    $\tr\,\tilde\M_i\lesssim \tr\,\M$ and $\|\tilde\M_i\|\lesssim\|\M\|$.
Using Burkholder's inequality \citep[e.g., see][]{hitczenko1990best} for the martingale difference
    sequence $X_i=(\E_i - \E_{i-1})[\tr\B\Q\mid\Ec]$, we get:
    \begin{align*}
      \E\big[|F(X) - \tilde F|^p\big]^{1/p}
      &=     \E\bigg[\Big(\sum_{i=1}^nX_i\Big)^p\bigg]^{1/p}
\leq
      Cp\cdot\Bigg(\E\bigg[\Big(\sum_{i=1}^n\E_{i-1}[X_i^2]\Big)^{p/2}\bigg]
      +\sum_{i=1}^n\E\big[X_i^p\big]\Bigg)^{1/p}
    \\
    &\lesssim\frac{p}{n}\cdot\Bigg(\Big(n ((\tr\,\M)^2+ K^4\|\M\|)\Big)^{p/2} +
      n(\tr\,\M+pK^{4}\|\M\|)^p\Bigg)^{1/p}
    \\
    &\lesssim p\big(\tr\,\M + K^4\|\M\|\big)\cdot\Big(\frac1{\sqrt n} + \frac{p}{n^{1-1/p}}\Big).
    \end{align*}
    Applying Markov's inequality with $p=\log(1/\delta)$, we obtain the
    desired result.
  \end{proof}

Note that in Lemma \ref{l:trace-functionals} we did not explicitly
state what is the scalar $\tilde F$ around which the trace functional
$F(\X)$ concentrates. Normally, one would expect this quantity to be
the expectation of $F(\X)$, however since this expectation may not exist,
we instead realy on the notion of a nearly-unbiased estimator. For the
trace functional we are interested in, i.e.,
$F(\X)=\tr\,\Sigmab(\X^\top\X)^{-1}$, a nearly-unbiased estimator has been
derived by \cite{less-embeddings}, and we use that result in the
following corollary.
  
  \begin{corollary}
    Suppose that $\X$ consists of rows $\x_1^\top,\x_2^\top,...,\x_n^\top$
    that are i.i.d. sampled from a $d$-variate distribution with
    covariance $\E[\x_i\x_i^\top]=\Sigmab$ such that
    $\Sigmab^{-1/2}\x_i$ satisfies the Hanson-Wright inequality
    \eqref{eq:hanson-wright} with $K=\tilde O(1)$ and $n\geq \tilde
    O(d)$. Then, with probability $1-\delta$: 
    \begin{align*}
\tr\Sigmab(\X^\top\X)^{-1} \approx_{1+\epsilon}
      \frac{d}{n-d}\quad\text{where}\quad
      \epsilon\leq\tilde O\bigg(\frac{\log1/\delta}{\sqrt n}\bigg).
    \end{align*}
  \end{corollary}
  \begin{proof}
    Here, we use the result from \cite{less-embeddings} which shows
    that $\Sigmab^{1/2}(\frac1n\X^\top\X)^{-1}\Sigmab^{1/2}$ is an $(\epsilon',\delta)$-unbiased
    estimator of $\frac n{n-d}\I$, where $\epsilon' = O(K^4\sqrt
    d/n)$. 
    Namely, this means that there is an event $\Ec$ with
    probability $1-\delta$ such that
    $\E[\Sigmab^{1/2}(\frac1n\X^\top\X)^{-1}\Sigmab^{1/2}\mid \Ec]\approx_{1+\epsilon}\frac n{n-d}\I$
    and conditioned on $\Ec$ we have $\Sigmab^{1/2} (\frac1n\X^\top\X)^{-1}\Sigmab^{1/2}\preceq
    O(1)\cdot \frac n{n-d}\I$. As a consequence, following the
    argument analogous to 
    Lemma~34 in \cite{less-embeddings}, we can easily show that if
    $F(\X)=\tr\Sigmab(\frac1n\X^\top\X)^{-1}$, then our
    scalar $\tilde F$ used in Lemma \ref{l:trace-functionals}
    also satisfies $\tilde F\approx_{1+\epsilon} \frac n{n-d}\cdot d$. So,
    combining this with Lemma \ref{l:trace-functionals}, and observing
    that in this case we have $\M = \I$ completes the proof.
  \end{proof}

  To conclude the proof of Lemma \ref{l:event}, we combine the
  corollary with the Hanson-Wright inequality, which states that, conditioned on $\X_{-i}^\top\X_{-i}$:
  \begin{align*}
    \Pr&\Big\{|\x_i^\top(\X_{-i}^\top\X_{-i})^{-1}\x_i-\tr\Sigmab(\X_{-i}^\top\X_{-i})^{-1}|\geq
    \epsilon\cdot \tr\Sigmab(\X_{-i}^\top\X_{-i})^{-1} \Big\}
    \\
    &\leq
    2\exp\Big(-c\min\Big\{\frac{\epsilon^2\cdot(\tr\Sigmab(\X_{-i}^\top\X_{-i})^{-1})^2}
    {K^4\|\Sigmab^{1/2}(\X_{-i}^\top\X_{-i})^{-1}\Sigmab^{1/2}\|_F^2},
      \frac{\epsilon\cdot\tr\Sigmab(\X_{-i}^\top\X_{-i})^{-1}}
      {K^2\|\Sigmab^{1/2}(\X_{-i}^\top\X_{-i})^{-1}\Sigmab^{1/2}\|}\Big\}\Big)
    \\
    &\leq \exp\Big(-c\epsilon^2\frac{\tr\Sigmab(\X_{-i}^\top\X_{-i})^{-1}}{K^4\|\Sigmab^{1/2}(\X_{-i}^\top\X_{-i})^{-1}\Sigmab^{1/2}\|}\Big).
  \end{align*}
Using that $K=\tilde O(1)$, as well as that with high probability we have $\tr\Sigmab(\X_{-i}^\top\X_{-i})^{-1}\approx \frac d{n-d}$
and $\|\Sigmab^{1/2}(\X_{-i}^\top\X_{-i})^{-1}\Sigmab^{1/2}\|=O(1/n)$,
we conclude that with high  probability, we also have:
  \begin{align*}
    \x_i^\top(\X_{-i}^\top\X_{-i})^{-1}\x_i\approx_{1+\epsilon}
    \frac{d}{n-d}\qquad\text{for}\qquad\epsilon = \tilde O(1/\sqrt d).
  \end{align*}
  By coupling a LESS embedding with the sub-gaussian design, we obtain
  the desired claim.

  \section{Constrained least squares: Proof of Theorem
  \ref{t:constrained}}
\label{a:constrained}
The claim can be shown by using our reduction from a LESS
embedding to a sub-gaussian embedding. Namely, let $\Z\Sigmabt^{1/2}$
be the sub-gaussian design that is coupled with the LESS embedding
$\sqrt n\S\A$ in Theorem~\ref{c:less-embeddings}. Recall that we
have $\Sigmabt\approx_{1+\delta}\Sigmab$. To simplify the
reduction, let us also define the following sketching matrix:
\begin{align*}
  \St = \frac1{\sqrt n}\big(\Z\Sigmabt^{1/2}\Sigmab^{-1/2}\U^\top + \G(\I-\U\U^\top)\big),
\end{align*}
where $\U=\A\Sigmab^{-1/2}$ and $\G$ is an $n\times N$ Gaussian matrix. It is easy to verify
that $\sqrt n\St\A = \Z\Sigmab^{1/2}=\sqrt n\S\A$ with probability
$1-\delta$, since $\I-\U\U^\top$ is the projection onto the complement
of the column-span of $\A$. Moreover, we have
$\E[\St^\top\St]\approx_{1+\delta}\I$, with each row of $\sqrt n\St$ being mean
zero and $\tilde O(1)$-sub-gaussian. Thus, it suffices to verify
that the analysis of \cite{pilanci2015randomized} still works when the
sketching matrix is only approximately isotropic. Note that the
difference here is very small since we can easily let $\delta\ll
\epsilon$. It suffices to show that their Proposition 1 can be
adapted. This proposition is itself a corollary of Theorem D from
\cite{mendelson2007reconstruction}. We can easily obtain a variant
that suits our setup.
\begin{proposition}
  Let $\sbt_1,...,\sbt_n$ be i.i.d. $N$-dimensional samples from a zero mean $K$-sub-gaussian
  distribution with covariance
  $\E[\sbt_i\sbt_i^\top]\approx_{1+\delta}\I$. For any subset
  $\mathcal Y\subseteq \mathcal S^{N-1}$, if $n\geq
  CK^2\mathbb W^2(\mathcal Y)/\epsilon^2$ and $\delta\leq c\epsilon$, then:
  \begin{align*}
    \sup_{y\in\mathcal Y}
    \Big|\y^\top\Big(\frac1n\sum_{i=1}^n\sbt_i\sbt_i^\top
    -\I\Big)\y\Big|\leq \epsilon,
  \end{align*}
  with probability at least $1-2\exp(-c\epsilon^2n/K^4)$.
\end{proposition}
The proposition follows from Theorem D of
\cite{mendelson2007reconstruction}, with functions $f_y(\s) =
\frac{\s^\top\y}{\|\y\|_{\M}}$, where $\M = \E[\sbt_i\sbt_i^\top]$.
Remainder of the analysis in \cite{pilanci2015randomized} stays
the same.

\section{Lower bound: Proof of Theorem \ref{t:lower}}
\label{a:lower}

Without loss of generality, assume that $N$ is a multiple of $d$
(otherwise we can pad the matrix with zeros), and consider matrix $\A$
consisting of stacked $d\times d$ identity matrices, scaled by $\sqrt{d/N}$
so that $\A^\top\A=\I$. It suffices to prove the result for $n=1$,  in which case the
sketching matrix $\S$ and the design matrix $\Z$ both consist of
single random vectors $\s$ and $\z$, respectively.

We will in fact show a strictly stronger claim that
$\|\z\|_{\psi_2}\geq \min\big\{1,c\sqrt{\frac{d/k}{\log d}}\big\}$, for some constant
$c>0$. This claim is stronger because, 
by Proposition \ref{p:euclidean}, the Hanson-Wright constant $K$ upper
bounds the sub-gaussian norm as follows: $\|\z\|_{\psi_2} =
O(K)$. Observe that since $\z$ is isotropic, it has to satisfy
 $\|\z\|_{\psi_2}\geq 1$, so it suffices to show 
 $\|\z\|_{\psi_2}\geq c\sqrt{\frac{d/k}{\log d}}$ for $k\leq d/C$, where $C$ is some
 absolute constant. 

 Let $\xbt = \A^\top\s = \frac1{\sqrt
   k}\sum_{i=1}^k\frac{r_i}{\sqrt{p_{I_i}}}\a_{I_i}$, and let
 $\zbt=\Sigmab^{1/2}\z$ be the random vector coupled with
 $\xbt$ so that there is an event $\Ec$ with probability at least
 $1/2$ such that, conditioned on $\Ec$, we have $\zbt=\xbt$.

 Next, in the joint probability space of $(\xbt,\zbt)$, for each
 $j=1,...,d$, define an event $\Ac_j$, which holds when exactly one of the
    vectors $\a_{I_1},...,\a_{I_k}$ has a non-zero $j$th coordinate. Also, let
 $\Ac=\Ac_1\vee...\vee\Ac_d$. Let us start by showing the
 $\Pr\{\Ac\}\geq 2/3$. Each row of $\A$ has only one non-zero
    coordinate, and each coordinate is equally represented among the
    rows, which means all of the leverage scores of $\A$ are the
    same. So, since the probability distribution $p$ is approximately 
    uniform, i.e., $p_i\approx_{O(1)} 1/N$ for all $i$,
    after sampling $k-1$ rows, we have at least $1- O(k/d)$ chance that
    the last sample will produce a row $\a_{I_k}$ with a non-zero in a
    coordinate we have not seen before. Choosing constant $C$
    apprioriately, we can ensure that this probability is at least
    $2/3$. Thus, we showed that $\Pr\{\Ac\}\geq 2/3$.

    Next, our goal is to show that one of the events $\Ac_j$ has a
    positive intersection with the event $\Ec$ (i.e., when $\xbt$ coincides
    with $\tilde \z$). We proceed as follows:
    \begin{align*}
      \sum_{j=1}^d\Pr\{\Ac_j\wedge\Ec\}
      \geq \Pr\{\Ac\wedge\Ec\}
      \geq \Pr\{\Ac\}-\Pr\{\neg\Ec\}
      \geq 2/3- 1/2 \geq 1/6,
    \end{align*}
    where we used the union bound. We conclude that there exists $j$
    such that $\Pr\{\Ac_j\wedge\Ec\}\geq1/(6d)$. Note that as long as
    $\Ac_j$ holds, then $|\xt_j|\geq c\sqrt{d/k}$  (here, $\xt_j$
 denotes the $j$th coordinate of $\xbt$), so:
    \begin{align*}
      \|\xbt\cdot\one_{\Ec}\|_{\psi_2}
      \geq \big\||\xt_j|\cdot\one_{\Ec}\big\|_{\psi_2}
      \geq \Big\|c\sqrt{d/k}\cdot\one_{\Ac_i\wedge\Ec}\Big\|_{\psi_2}
      =\Omega\Big(\sqrt{\tfrac{d/k}{\log d}}\,\Big),
    \end{align*}
    where we used a simple lower bound for the sub-gaussian norm of a
    Bernoulli random variable. This immediately implies the same bound
    for $\tilde\z$ since
    $\|\tilde\z\|_{\psi_2}\geq\|\xbt\cdot\one_{\Ec}\|_{\psi_2}$. We
    can finally return to $\z=\Sigmabt^{-1/2}\tilde\z$. Using
    the fact that $\Sigmabt\preceq O(1)\cdot\I$, we get:
    \begin{align*}
    \|\z\|_{\psi_2}
    &=\sup_{\v:\|\v\|=1}\|\v^\top\Sigmabt^{-1/2}\tilde\z\|_{\psi_2}
=\sup_{\v:\|\v\|=1}\|\v^\top\Sigmabt^{-1/2}\|\bigg\|\frac{\v^\top\Sigmabt^{-1/2}}{\|\v^\top\Sigmabt^{-1/2}\|}\tilde\z\bigg\|_{\psi_2}
    \\
    &=\Omega(1)
      \sup_{\v:\|\v\|=1}\bigg\|\frac{\v^\top\Sigmabt^{-1/2}}{\|\v^\top\Sigmabt^{-1/2}\|}\tilde\z\bigg\|_{\psi_2}
= \Omega\big(\|\zbt\|_{\psi_2}\big)=\Omega\Big(\sqrt{\tfrac{d/k}{\log d}}\,\Big).
    \end{align*}

\section{Proof of Lemma \ref{l:sub-exponential}}
\label{a:sub-exponential}

Let $\u = \Sigmab^{-1/2}\x$. We will use the following simple
decomposition:
\begin{align}
  (\u\u^\top-\I)^p = (\|\u\|^2-1)^{p-1}\u\u^\top -
  (\u\u^\top-\I)^{p-1}.
  \label{eq:moment-decomposition}
\end{align}
Using the fact that $\|\u\|$ is $M$-sub-gaussian, for any $t>0$ we have that:
\begin{align*}
  \Pr\big\{\|\u\|^2\geq t\big\}\leq \exp(-ct/M^2),
\end{align*}
for some absolute constant $c$. Define the following event:
$\Ec=\big[\|\u\|^2\leq L\big]$ where $L=CM^2(p+3\log M)$.  We conclude that
$\Pr\{\neg\Ec\}\leq\exp(-c'(p+2\log M))\leq M^{-2}\exp(-c'(p+\log M))$. Next, we
use this to bound the first term from \eqref{eq:moment-decomposition}:
\begin{align*}
  \big\|\E\big[(\|\u\|^2-1)^{p-1}\u\u^\top\big]\big\|
  &\leq \big\|\E\big[|\|\u\|^2-1|^{p-1}\u\u^\top \one_{\Ec}\big]\big\|
    +
    \big\|\E\big[|\|\u\|^2-1|^{p-1}\u\u^\top \one_{\neg\Ec}\big]\big\|
  \\
  &\leq L^{p-1}\|\E[\u\u^\top]\|
    + \E\big[\|\u\|^{2p}\one_{\neg \Ec}\big]
  \\
  &\leq L^{p-1} + \int_0^\infty
    px^{p-1}\Pr\big\{\|\u\|^2\one_{\neg\Ec}>x\big\}dx
  \\
  &\leq L^{p-1}
    +  L^p\cdot M^{-2}\exp(-c(p+\log M)) +
    \int_{L}^\infty px^{p-1}\exp(-cx/M^2)\,dx
  \\
  &\leq C'L^{p-1} + (C'M^2p)^p\exp(-c'L/M^2)
  \\
  &\leq (C''L)^{p-1},
\end{align*}
where $C$, $C'$ and $C''$ all denote absolute constants. Thus, we
obtain that $\|\E[(\u\u^\top-\I)^p]\| \leq (C''L)^{p-1} + 
\|\E[(\u\u^\top-\I)^{p-1}]\| \leq (C'''L)^{p-1}$ (by expanding the
recursion). This completes the proof of Lemma \ref{l:sub-exponential}.

Here, we also establish a closely related result used in the proof of
Theorem \ref{t:main}.
\begin{lemma}\label{l:sub-exp-scalar}
  There is an absolute constant $C>0$ such that any $d$-dimensional
  random vector $\x$ with covariance $\E[\x\x^\top]=\Sigmab$, where $\|\Sigmab^{-1/2}\x\|$ is $M$-sub-gaussian, satisfies the
  following sub-exponential moment bound for any psd matrix
  $\B$:
  \begin{align*}
    \E\Big[\big|\|\B^{1/2}\Sigmab^{-1/2}\x\|^2 - \tr\B\big|^p\Big]
    \leq (CM^2(p+\log
   M))^{p-1}\|\B\|^{p-2}\|\B\|_F^2.
  \end{align*}
\end{lemma}
\begin{proof}
  We proceed similarly to the proof of Lemma \ref{l:sub-exponential},
  letting $L=CM^2(p+\log M)$ and $\Ec=[\|\Sigmab^{-1/2}\x\|^2\leq L]$:
  \begin{align*}
    \E\Big[\|\B^{1/2}\Sigmab^{-1/2}\x\|^{2p}\Big]
    &\leq
      \E\Big[\|\Sigmab^{-1/2}\x\|^{2(p-1)}\x^\top\Sigmab^{-1/2}\B^p\Sigmab^{-1/2}\x
      \Big]
    \\
    &\leq L^{p-1}\tr\B^p +
      \|\B\|^p\cdot\E\big[\one_{\neg\Ec}\|\Sigmab^{-1/2}\x\|^{2p}\big]
    \\
    &\leq (CL\|\B\|)^{p-2}L\|\B\|_F^2,
  \end{align*}
  where we used the integral bound from the proof of
  Lemma \ref{l:sub-exponential} and that $\tr\B^p\leq
  \|\B\|^{p-2}\|\B\|_F^2$. Finally, note that
  \begin{align*}
    \E\Big[\big|\|\B^{1/2}\Sigmab^{-1/2}\x\|^2 - \tr\B\big|^p\Big]
    &\leq 2^p\bigg(  \E\Big[\big|\|\B^{1/2}\Sigmab^{-1/2}\x\|^2\Big]  +
    (\tr\B)^p\bigg)
    \\
    &\leq 2^p(L\|\B\|)^{p-2}L\|\B\|_F^2 + 2^p(d\|\B\|)^{p-2}\|\B\|_F^2
    \\
    &\leq (C'L\|\B\|)^{p-2}L\|\B\|_F^2,
  \end{align*}
  where we used the fact that $L\geq M^2\geq d$.
\end{proof}


\section{Leave-one-out cross-validation formula: Proof of Lemma
  \ref{l:cv}}
\label{a:cv}
First, note that, assuming
$(\A^\top\S_{-i}\S_{-i}\A)^{-1}\preceq O(1)\cdot(\A^\top\A)^{-1}$,
we have:
\begin{align*}
  \big\|\argmin_{\w}\|\S_{-i}(\A\w-\b)\|\,\big\|
  &=
    \|(\A^\top\S_{-i}^\top\S_{-i}\A)^{-1}\A^\top\S_{-i}^\top\S_{-i}\b\|
  \\
  &\leq
    \|(\A^\top\S_{-i}^\top\S_{-i}\A)^{-1}\A^\top\S_{-i}^\top\|\cdot
    \|\S_{-i}\b\|
  \\
  &\leq O(1) \sqrt{\|(\A^\top\A)^{-1}\|}\cdot O(\sqrt{Nk})\|\b\|\leq \poly(N,\kappa(\A),\|\A\|,\|\b\|),
\end{align*}
where we used that $\|\S_{-i}\|^2= O(Nk)$ for LESS embeddings
satisfying the assumptions from Theorem~\ref{t:main-app}. The same bound holds
for the unconstrained version of $\wbh$. Thus, it follows that, for
sufficiently large $D$, each 
$\wbh^{-i}$, as well as $\wbh$, can be computed using the standard
formulas for the unconstrained least squares, without restricting to $\mathcal R^d=[-D,D]^d$. 

\paragraph{Proof of the shortcut formula.} This is a standard
derivation, which we include for completeness. As shorthands, we let $\X=\S\A$, $\y=\S\b$, and $\ell_i=\x_i^\top(\X^\top\X)^{-1}\x_i$. From
  Sherman-Morrison (Lemma \ref{l:rank-one}), we have:
  \begin{align*}
    (\X_{-i}^\top\X_{-i})^{-1}\x_i =
    \frac{(\X^\top\X)^{-1}\x_i}{1-\x_i^\top(\X^\top\X)^{-1}\x_i} = \frac{(\X^\top\X)^{-1}\x_i}{1-\ell_i},
  \end{align*}
  where $\X_{-i}$ is $\X$ without the $i$th row.
  Next, we use that to obtain:
  \begin{align*}
    \wbh^{-i}
     &= (\X_{-i}^\top\X_{-i})^{-1}\X^\top\y -
      (\X_{-i}^\top\X_{-i})^{-1}\x_iy_i
    \\
    &=\bigg((\X^\top\X)^{-1} +
      \frac{(\X^\top\X)^{-1}\x_i\x_i^\top(\X^\top\X)^{-1}}{1-\ell_i}\bigg)\X^\top\y
      - \frac{(\X^\top\X)^{-1}\x_iy_i}{1-\ell_i}
    \\
    &=\wbh +
      \frac{(\X^\top\X)^{-1}\x_i}{1-\ell_i}
      \Big(\x_i^\top(\X^\top\X)^{-1}\X^\top\y - y_i\Big)
=\wbh + \frac{\x_i^\top\wbh-y_i}{1-\ell_i}(\X^\top\X)^{-1}\x_i.
  \end{align*}
  Finally, we plug this into the $i$th component of the leave-one-out cross-validation estimate:
  \begin{align*}
    \x_i^\top\wbh^{-i}-y_i = \x_i^\top\wbh +
    \frac{\x_i^\top\wbh-y_i}{1-\ell_i}\cdot \ell_i - y_i = \frac{\x_i^\top\wbh-y_i}{1-\ell_i}.
  \end{align*}
Squaring and summing the components concludes the proof.

\end{document}